\documentclass{article}

\usepackage{microtype}
\usepackage{graphicx}
\usepackage{subcaption}
\usepackage{booktabs}
\usepackage{hyperref}

\usepackage[preprint]{icml2026}

\usepackage{amsmath}
\usepackage{amssymb}
\usepackage{mathtools}
\usepackage{amsthm}

\usepackage[capitalize,noabbrev]{cleveref}
\crefname{assumption}{Assumption}{Assumptions}

\theoremstyle{plain}
\newtheorem{theorem}{Theorem}[section]

\newtheorem{lemma}[theorem]{Lemma}

\newtheorem{example}{Example}
\theoremstyle{definition}

\newtheorem{assumption}[theorem]{Assumption}
\theoremstyle{remark}
\newtheorem{remark}[theorem]{Remark}

\usepackage[textsize=tiny]{todonotes}

\usepackage{tikz}
\usetikzlibrary{shapes.geometric}
\usepackage{multirow}
\usepackage[table]{xcolor}
\usepackage{soul}

\icmltitlerunning{GoT: Robust Information Seeking with LLMs using Game Theory}

\begin{document}

\twocolumn[
  \icmltitle{Game of Thought: Robust Information Seeking \texorpdfstring{\\}
\ with Large Language Models Using Game Theory}

  \icmlsetsymbol{equal}{*}

  \begin{icmlauthorlist}
    \icmlauthor{Langyuan Cui}{nussoc}
    \icmlauthor{Chun Kai Ling}{nussoc}
    \icmlauthor{Hwee Tou Ng}{nussoc}

  \end{icmlauthorlist}

  \icmlaffiliation{nussoc}{Department of Computer Science, National University of Singapore, 13 Computing Drive, Singapore 117417}

  \icmlcorrespondingauthor{Langyuan Cui}{langyuan.c@u.nus.edu}

  \icmlkeywords{Game Theory, LLMs, ICML}

  \vskip 0.3in
]

\printAffiliationsAndNotice{}  

\begin{abstract}
Large Language Models (LLMs) are increasingly deployed in real-world scenarios where they may lack sufficient information to complete a given task. In such settings, the ability to actively seek out missing information becomes a critical capability. Existing approaches to enhancing this ability often rely on simplifying assumptions that degrade \textit{worst-case} performance. This is an issue with serious implications in high-stakes applications.
In this work, we use the game of Twenty Questions to evaluate the information-seeking ability of LLMs. We introduce and formalize its adversarial counterpart, the Strategic Language Search (SLS) problem along with its variants as a two-player zero-sum extensive form game. We propose Game of Thought (GoT), a framework that applies game-theoretic techniques to approximate a Nash equilibrium (NE) strategy for the restricted variant of the game. Empirical results demonstrate that our approach consistently improves worst-case performance compared to (1) direct prompting-based methods and (2) heuristic-guided search methods across all tested settings.

\end{abstract}

\section{Introduction}
\label{sec:intro}

Large language models (LLMs) are increasingly being deployed in high-stakes environments like planning \cite{zhang-etal-2024-ask}, medical diagnosis \cite{Li2024MediQQL}, as well as other tasks exhibiting partial observability \cite{Li2025QuestBenchCL}. In such environments, the LLM may not have sufficient information to complete its assigned task. This necessitates an information-seeking process, typically via the use of clarification questions.
To quantitatively assess information-seeking ability, we turn to the \textit{Game of 20 Questions} \cite{Siegler1977TwentyQ}. Here, an item is first chosen from a known set. The player then sequentially asks up to twenty Yes/No questions to identify the item using as few questions as possible. 

Information seeking often requires some lookahead. Well-known methods involve a combination of prompting and LLM-based reasoning, such as Self-Consistency (SC) \cite{Wang2022SelfConsistencyIC} and Tree of Thought (ToT) \cite{Yao2023TreeOT}. More recently, \citet{Hu2024UncertaintyOT} propose Uncertainty of Thought (UoT), which utilizes tree search to \textit{explicitly model and optimize} for the information gained from clarification questions. However, they assume that the item is chosen uniformly at random. This is unlikely to hold in the real world, either because (i) the context itself does not naturally admit a probabilistic interpretation \cite{Jain2016WhyID}, or (ii) even if a distribution exists, it might not be easily obtainable, and is likely non-uniform \cite{choudhury2025bedllmintelligentinformationgathering}. Particularly for high-stakes environments, we argue that one should assume the \textit{worst case}. That is, the item is chosen \textit{adversarially}, and our goal is to utilize a questioning strategy that optimizes the worst-case performance.

Motivated by this, we propose Game of Thought (GoT), a game-theoretic framework designed to handle such scenarios. 
We model information-seeking process as a two-player, zero-sum game with imperfect information. In this formulation, the item (distribution) is assumed to be selected by an adversary seeking to impede the information seeker. GoT optimizes for the worst-case performance by approximating the Nash equilibrium (NE) of this game, avoiding any need for strong prior assumptions on item distribution. 
\paragraph{Contributions} We \textbf{(i)} formulate the Strategic Language Search (SLS) problem and its variants, \textbf{(ii)} establish game theoretic solutions that solve SLS and its variants optimally, highlighting these strategies are necessarily randomized, \textbf{(iii)}  propose GoT to approximate the NE in SLS-restricted variant, \textbf{(iv)} empirically demonstrate that in various settings, GoT is superior to (a) pure prompting-based and (b) UoT-based methods in the worst-case.

\section{Related Work}
\paragraph{Uncertainty of Thought} GoT can be regarded as an extension of \textit{Uncertainty of Thought} (UoT) by \citet{Hu2024UncertaintyOT}. The authors propose performing depth limited search over possible questions asked, seeking to maximize expected information gain under the assumption that the items are \textit{chosen uniformly at random}. GoT performs similar explicit lookahead, but obviates this assumption via game-theoretic approaches to optimize for the worst-case item chosen. 

\paragraph{Natural Language in Games}
We formulate games that implicitly consider natural language as the action space. Language is often used in games for communication between players as means to negotiate \cite{Kramr2022NegotiationAH}, share information \cite{Ri2022TheDO}, and coordinate \cite{Bishop2020CHAOPT}. Studying language in this context poses an interesting challenge due to the difficulty of providing a well-defined notion of optimality and in analyzing the effect of these utterances. These communications resemble signals in signaling games \cite{Lewis1969-LEWCAP-4,Farrell1996CT,GINTIS2001103}. However, these games typically consider a finite set of signals, unlike communications in natural language. Prior studies have attempted to reconcile this difference via various techniques such as sentiment analysis \cite{Capraro_2024,Houser2011}, or by sampling from LLMs conditioned on the context of the game \cite{Gemp2024SteeringLM,kempinski2025game}, leveraging its reasoning ability to filter out irrelevant utterances while using game theory inspired techniques to design strategy.

\paragraph{Information-seeking ability of LLMs}
Prior studies examined LLM information seeking ability in the context of ambiguous queries or under-specified tasks. In these settings, the LLM is initially missing details needed to complete the task and must pose clarifying questions to identify either user intent \cite{Kuhn2022CLAMSC,zhang-etal-2024-ask,xu-etal-2019-asking} or elicit user preferences \cite{Handa2024BayesianPE,Li2023ElicitingHP}.
Framed this way, the problem becomes one of reasoning about which additional information would make the task well specified. Exploration based prompting has been shown to help (SC, ToT \cite{Wang2022SelfConsistencyIC,Yao2023TreeOT,choudhury2025bedllmintelligentinformationgathering}), while other approaches select queries using heuristics such as expected information gain \cite{Grand2024LooseLS,Piriyakulkij2023ActivePI}, including UoT \cite{Hu2024UncertaintyOT}.
Benchmarks evaluating information-seeking ability of LLMs have been constructed in the context of underspecified tasks \cite{Li2025QuestBenchCL}, medical diagnosis \cite{Li2024MediQQL, chen2024codinterpretablemedicalagent}, and troubleshooting \cite{raghu-etal-2021-end}, demonstrating the importance of this capability in actual application.

\section{Problem Formulation}
While the Game of 20 Questions is not strictly adversarial, we present an adversarial variant which we refer to as Strategic Language Search (SLS) to enable utilization of game theoretic methods to optimize for the worst-case. A SLS game is played over a finite set $\mathcal{S}$ of $n$ distinct items between two competitive (zero-sum) players, the \textit{Item Chooser} and \textit{Questioner}. We will assume a countable set of binary (i.e., true-false) questions $\mathcal{Q}$. Each question is uniquely specified by a finite-length natural language string. For every item $s \in \mathcal{S}$, we denote by $f: \mathcal{Q} \times \mathcal{S} \rightarrow \{0, 1\}$ the \textit{answer} to $q\in \mathcal{Q}$ for some $s \in \mathcal{S}$. A natural but degenerate $\mathcal{Q}$ is one which we denote by $\mathcal{Q}_\infty$. A possible textual representation of $\mathcal{Q}_\infty$ would be the set of questions stated in boolean form ``Is the Item 1 or Item 3, but not Item 2 or Item 4?''

A SLS is defined by $(\mathcal{S}, \mathcal{Q}, f)$ and proceeds in two phases. In the first phase, the Item Chooser privately selects a single item $s^* \in \mathcal{S}$. In the second phase, the Questioner elicits information about $s^*$ by sequentially asking and obtaining answers to a series of questions $q_1, a_1, \dots, q_T, a_T$ where $q_t \in \mathcal{Q}$ and $a_t = f(q_t, s^*)$, i.e., the answer to $q_t$ for $s^*$. 

The history of questions and answers up to and including time $t \geq 0$ is given by $H_t = (Q_t, A_t)$, where $Q_t = (q_1, \dots, q_t), A_t = (a_1,\dots,a_t)$ such that $Q_t(\tau)=q_\tau$ and $A_t(\tau)=a_\tau$. 
Given history $H$, we denote by $H'(q, a)$ the new history when a new question $q$ is asked with answer $a$.
The Questioner selects its questions online, i.e., $q_t$ may depend on its observed history $H_{t-1}$.
The length of history $H$ is denoted by $|H|$. 

The set of items consistent with some observed history $H=(Q, A)$ is denoted by
\begin{align*}
 S(H) = \{ s \in \mathcal{S} | \forall \tau \in [|H|], f(Q(\tau), s) = A(\tau) \} \subseteq \mathcal{S}.
\end{align*}
Note that $S(H)$ is never empty, as it must contain at least $s^*$.
The game continues until the Questioner can identify $s^*$ with certainty, i.e., $|S(H)| = 1$, after which the game ends with the Item Chooser (resp. Questioner) incurring a reward (resp. cost) of $|H|$.
The objective for the Questioner is to find a strategy (or equivalently, policy) that minimizes its expected cost regardless of how the Item Chooser chooses $s^*$, optimizing its worst-case performance. This corresponds to the \textit{NE} of a zero-sum game and will typically be randomized. We will define the solution concept formally after stating several assumptions and variants.

\begin{assumption}
   Each question-item pair $(q, s) \in \mathcal{Q} \times \mathcal{S}$ has a unique answer $f(q, s)$ independent of history.
   \label{ass:unique_ans}
\end{assumption}
\begin{assumption}
    After fixing $s^*$, the Item Chooser cannot lie with regard to its answers to questions.     
    \label{ass:no_lying}
\end{assumption}
\begin{assumption}
    $\mathcal{S}$, $\mathcal{Q}$ and $f$ are common-knowledge.
    \label{ass:common_knowledge}
\end{assumption}
\begin{assumption}
    Constant time oracle access to $f$.
    \label{ass:oracle-access}
\end{assumption}
\begin{assumption}
    For every pair of distinct items $s, s' \in \mathcal{S}$, there exists some $q \in \mathcal{Q}$ where $f(q, s) \neq f(q, s')$.
    \label{ass:can_always_split}
\end{assumption}
\cref{ass:unique_ans,ass:no_lying} are implicit assumptions baked into the definition of $f$ while \cref{ass:can_always_split} is a technical assumption imposed on $\mathcal{Q}$ and $f$. It ensures that there exists a strategy that allows the Questioner to find the correct item in finite time. Note that it is trivially satisfied when $\mathcal{Q}$ includes ``identity'' questions such as ``is the item `$s$'?''. 

\begin{example}
    Consider the SLS with $\mathcal{S}$ and $\mathcal{Q}$: 
    \begin{align*}
    \left\{
    \begin{aligned}
    &s^{(1)}: \text{\upshape`Oppenheimer'} \\
    &s^{(2)}: \text{\upshape`Alan Turing'} \\
    &s^{(3)}: \text{\upshape`A Beautiful Mind'} 
    \end{aligned}
    \right\},
    \left\{
    \begin{aligned}
    &q^{(1)}: \text{\upshape`Related to codes?'} \\
    &q^{(2)}: \text{\upshape`Is it a movie?'} \\
    &q^{(3)}: \text{\upshape`Is it a person?'} 
    \end{aligned}
    \right\}
    \end{align*}
    Then $f(q^{(i)},s^{(j)})=0$ if $j = i$, 1 otherwise.
    \label{ex:circular}
\end{example}

It is easy to verify that \cref{ass:can_always_split} holds in Example~\ref{ex:circular}. By selecting $q^{(1)}$ and $q^{(2)}$ in sequence, one can guarantee that the Questioner asks no more than $2$ questions. The best the Item Chooser can do against this strategy is to select $s^*=s^{(2)}$ or $s^{(3)}$, since doing so guarantees that there are always two items consistent after the first question, i.e., $|S(H_1)|=2$. However, the Questioner can do better by choosing the first question uniformly at random. This guarantees that regardless of the choice of $s^*$, there is $1/3$ probability that the item can be determined with exactly one question, thus the expected number of questions asked is $1/3 \cdot 1 + 2/3 \cdot 2 < 2$. 
This strategy minimizes the worst-case cost regardless of what $s^*$ the Item Chooser chose. In this example, we are lucky that the optimal strategy for the Questioner is symmetric due to the circular symmetric nature of $\mathcal{Q}$ and $f$. This is not true for more general $\mathcal{Q}, f$.

So far, SLS has been presented mostly as a combinatorial problem. Indeed, when $\mathcal{Q}$ is finite, we have:
\begin{theorem}
Given a known (deterministic) $s^*$, deciding if there exists a sequence of $k$ questions $Q \subseteq \mathcal{Q}$ such that $S(H)=\{ s^* \}$ is NP-complete.
\label{thm:np-complete-br-sls}
\end{theorem}
The reduction is straightforward from set-cover and included in the appendix. In game theoretic parlance, this implies that the \textit{best-response} of the Questioner to any deterministic Item Chooser strategy is hard to compute. 
\begin{theorem}
    Let $(\mathcal{S}, \mathcal{Q}_\infty, f)$ be a SLS, where $|\mathcal{S}|=2^k$ for some $k\in\mathbb{Z}^+$. An ``even-split'' strategy where $\forall t, q_{t+1}$ is selected such that
    $
        |S(H_t'(q_{t+1}, 0))| =  |S(H_t'(q_{t+1}, 1))|
    $
    minimizes the number of questions asked in the worst-case for the Questioner and costs exactly $k$ questions.
    \label{thm:even-split-opt-sls}
\end{theorem}
Theorem \ref{thm:even-split-opt-sls} states that in the special case where $\mathcal{Q}=\mathcal{Q}_\infty$ the problem becomes easy, agreeing with the intuition that splitting the items evenly is optimal. We show in the appendix that the even-split strategy is optimal in UoT's setting where $s^*$ is chosen uniformly at random \cite{Hu2024UncertaintyOT}; in fact, they constitute a NE in game setting. This justifies their implicit approach of maximizing entropy loss and usage of LLM prompts that encourage even-splits.

\paragraph{SLS-Restricted (SLSR)} Vanilla SLS can be quite cumbersome to reason about for larger $n$ and $\mathcal{Q}$. As such, we propose a SLSR, a restricted variant where questions are generated based on the remaining items $S(H)$. 
A SLSR is formally given by $(\mathcal{S}, \mathcal{Q}, f, g)$, where $g: 2^\mathcal{S} \backslash \phi \rightarrow 2^{\mathcal{Q}}\backslash\phi$ is a set function taking a nonempty set of items $S \subseteq \mathcal{S}$ and outputs a nonempty set of no more than $m \geq 1$ questions; here $m$ is a parameter associated with $g$. The rules and payoffs in SLSR are identical to SLS, except that the Questioner is restricted to selecting $q_t \in g(S(H_{t-1}))$. To ensure that progress can always be made, we impose a stronger version of Assumption~\ref{ass:can_always_split} and require that $g$ outputs at least one question that strictly reduces the set of consistent items. Note that in a SLSR game, an optimal Questioner can perform no better than in the corresponding SLS game. Thus its performance in the SLSR game can be seen as an underestimation of its performance in SLS. This design is motivated by our goal of examining the worst-case performance.
\begin{assumption}
    For every subset $S\subseteq\mathcal{S}$, there exists some $q\in g(S)$ and a pair of distinct items  $s, s'\in S$ where $f(q, s) \neq f(q, s')$
    \label{ass:can_always_progress}
\end{assumption}

\paragraph{Weighted-SLS (WSLS)} In real-world scenarios, certain items may carry greater importance and should be prioritized during the information-seeking process. For instance, in medical diagnosis, delays in identifying life-threatening conditions can have more severe consequences. 
Hence, we propose weighted-SLS, a variant of SLS defined by $(\mathcal{S}, \mathcal{Q}, f, w)$, where $w:\mathcal{S}\rightarrow \mathbb{R}^+$. The cost incurred by the Questioner with $s^*$ as the selected item is redefined to be $w(s^*) \cdot |H|$, imposing greater penalties for prolonged interactions involving high-weight items. Weighted SLSR (WSLSR) is defined in a similar manner by $(\mathcal{S},\mathcal{Q},f,g,w)$. The corresponding weighted variant of the non-adversarial game of Twenty Questions is defined similarly.

\begin{example}
    SLSR $(\mathcal{S}, \mathcal{Q}, f, g)$ with $\mathcal{S}, \mathcal{Q}, f$ identical to Example~\ref{ex:circular}, $g(S) = \{ q^{(1)}, q^{(2)}\}$ when $S=\mathcal{S}$ and $\mathcal{Q}$ otherwise. 
    \label{ex:circular-restricted}
\end{example}
In Example~\ref{ex:circular-restricted}, the first question is restricted to be either $q^{(1)}$ or $q^{(2)}$. Regardless of how the Questioner randomizes between them, the Item Chooser can always select $s^*=s^{(3)}$; this guarantees that $a_1=1$ and hence $|S(H_1)| = 2$. Thus, the Questioner requires at least one additional question, taking 2 questions in total.
\begin{example}
    WSLS $(\mathcal{S}, \mathcal{Q}, f, w)$ with $\mathcal{S}, \mathcal{Q}, f$ identical to Example~\ref{ex:circular}, $w(s^{(1)})=3, w(s^{(2)})=w(s^{(3)})=2$.
    \label{ex:circular-weighted}
\end{example}
In Example~\ref{ex:circular-weighted}, Questioner can select $q^{(1)}, q^{(2)}, q^{(3)}$ with probability $3/4, 1/8, 1/8$ as $q_1$. This ensures the expected cost is no greater than $15/4$. In contrast, selecting $q_1$ uniformly incurs a cost of $5$ when $s^*=s^{(1)}$ is chosen.

\begin{remark}
    In SLS and its variants, the game only ends when the Questioner is perfectly sure of the $s^*$. An alternative formulation permits the Questioner to explicitly guess the item at any stage, with penalties for incorrect guesses. 
    We chose the current formulation since it captures the constraints in high-stakes environments better, e.g., in medical diagnosis, one wishes to rule out all other possibilities before moving on to treatment. Roughly speaking, our formulation penalizes incorrect guesses with ``infinite'' penalty.
\end{remark}

\paragraph{Defining SLS via Large Language Models.} 
For many applications, $\mathcal{S}$ is too large for us to realistically hand-curate a set of $\mathcal{Q}$, noting that we would like to avoid ``unnatural'' questions that involve long, complicated boolean expressions. 
Therefore, in this paper we define $\mathcal{Q}$, $f$, and $g$ implicitly using Large Language Models (LLM). Specifically, $\mathcal{Q}$ is the set of questions that a LLM can propose asking based on certain prompts, and $f$ is the response of a LLM to a prompt of $q$ when asked about $s^*$. For instance, \texttt{Is Oppenheimer a movie?}. We wish to emphasize that $f$ can be implemented by any oracle, such as human experts. We chose to implement it via LLMs for the ease of experimentation. Furthermore, $g$ in SLSR and WSLSR can be defined as the output of a LLM with respect to a prompt asking for $m\geq 1$ questions after defining the rules of SLS and $S(H)$, the items remaining. An example of this is \texttt{Currently the set of possible items is \{Oppenheimer, Alan Turing\}. Propose m best questions.} 

\begin{assumption}
   The LLM for $f$ makes no mistakes when used as an oracle for $f(q, s^*)$.
   \label{ass:reliable_llm}
\end{assumption}
Assumption~\ref{ass:reliable_llm} and \ref{ass:unique_ans} enable the practical usage of LLMs as a black-box for defining SLS. For most LLMs, Assumption~\ref{ass:oracle-access} is also satisfied, though the constant factor may be large. We discuss all assumptions made in the appendix.

\section{SLS as Zero-sum Extensive Form Games} 
A vanilla SLS $(\mathcal{S}, \mathcal{Q}, f)$ may be expressed as a two-player zero-sum extensive form game (EFG) with imperfect information. EFGs are rooted game trees, where each vertex corresponds to the game's entire history. Edges represent actions available to the player-to-move at that vertex. EFGs are endowed with information sets (infosets), which partition vertices belonging to the same player;
vertices in the same infoset are indistinguishable to the player owning them. This captures the notion of imperfect information. 

\begin{figure}[ht]
\begin{tikzpicture}[
  level distance=12mm,
  every node/.style={circle,draw,minimum size=2mm,inner sep=0pt},
  level 1/.style={sibling distance=28mm},
  level 2/.style={sibling distance=10mm},
  level 3/.style={sibling distance=5mm},
  invertednode/.style={
    regular polygon,
    regular polygon sides=3,
    draw,
    minimum size=3mm,
    inner sep=0pt,
    rotate=180,
  },
]
\usetikzlibrary{trees,shapes.geometric}

\node[regular polygon, regular polygon sides=3, draw, fill=black, minimum size=3mm, inner sep=0pt] {}
  child { node[invertednode] (v1){}
    child[edge from parent/.style={draw=red}] { node[draw=none] {1}
    }
    child[edge from parent/.style={draw=blue}] { node[invertednode] (v13a) {}
      child[edge from parent/.style={draw=red}] { node[draw=none] {2} }
      child[edge from parent/.style={draw=green}] { node[draw=none] {2} }
    }
    child[edge from parent/.style={draw=green}] { node[invertednode] (v12a){}
      child[edge from parent/.style={draw=red}] { node[draw=none] {2} }
      child[edge from parent/.style={draw=blue}] { node[draw=none] {2} }
    }
    edge from parent node[draw=none,pos=0.5,above] {$s^{(1)}$}
  }
  child { node[invertednode] (v2){}
    child[edge from parent/.style={draw=red}] { node[invertednode] (v23a){}
      child[edge from parent/.style={draw=blue}] { node[draw=none] {2} }
      child[edge from parent/.style={draw=green}] { node[draw=none] {2} }
    }
    child[edge from parent/.style={draw=blue}] { node[draw=none] {1}
    }
    child[edge from parent/.style={draw=green}] { node[invertednode] (v12b){}
      child[edge from parent/.style={draw=red}] { node[draw=none] {2} }
      child[edge from parent/.style={draw=blue}] { node[draw=none] {2} }
    }
    edge from parent node[draw=none,pos=0.5,left] {$s^{(2)}$}
  }
  child { node[invertednode] (v3){}
    child[edge from parent/.style={draw=red}] { node[invertednode] (v23b){}
      child[edge from parent/.style={draw=blue}] { node[draw=none] {2} }
      child[edge from parent/.style={draw=green}] { node[draw=none] {2} }
    }
    child[edge from parent/.style={draw=blue}] { node[invertednode] (v13b) {}
      child[edge from parent/.style={draw=red}] { node[draw=none] {2} }
      child[edge from parent/.style={draw=green}] { node[draw=none] {2} }
    }
    child[edge from parent/.style={draw=green}] { node[draw=none] {1}
    }
    edge from parent node[draw=none,pos=0.5,above] {$s^{(3)}$}
  };
\draw[dashed] (v1) -- (v2);
\draw[dashed] (v2) -- (v3);
\draw[dashed] (v23a) to[out=30,in=150] (v23b);
\draw[dashed] (v13a) to[out=30,in=150] (v13b);
\draw[dashed] (v12a) to[out=30,in=150] (v12b);
\end{tikzpicture}
\caption{EFG representation of Example~\ref{ex:circular}. 
The root node \tikz[baseline=-0.65ex]{\node[regular polygon,regular polygon sides=3,draw,minimum size=3mm,inner sep=0pt,rotate=0,fill=black] {};} belongs to the Item Chooser and the other nodes 
\tikz[baseline=-1ex]{\node[regular polygon,regular polygon sides=3,draw,minimum size=3mm,inner sep=0pt,rotate=180] {};}
belong to the Questioner. Edges in black correspond to the choice of item $s^*$. Colored edges in \textcolor{red}{red}, \textcolor{blue}{blue}, and \textcolor{green}{green} refer to questions $q^{(1)}$, $q^{(2)}$, and $q^{(3)}$ respectively. Vertices connected by dotted lines belong to the same infoset. Payoffs (resp. costs) to the Item Chooser (resp. Questioner) are shown in the leaves. We omit edges (and their descendants) where the same question is asked more than once; these are never optimal and correspond to dominated actions.}
\label{fig:efg-circular}
\end{figure}

In the EFG formulation, the Item Chooser only takes an action at the root, after which the Questioner asks up to $n-1$ questions in sequence.\footnote{This keeps the game tree finite, and is justified since asking the same question more than once is suboptimal.} Every state apart from the root can be uniquely identified by the tuple $(s^*, H)$ which describes item $s^*$ chosen and a (potentially empty) history of past questions and answers.
Leaf vertices are those where either (i) $|H|=n-1$, where the maximum number of questions are asked, or (ii) $|S(H)| = 1$, where $s^*$ is identified. At a leaf, the reward (resp. cost) to the Item Chooser (resp. Questioner) is $|H|$.
All non-leaf vertices with the same history $H$ belong to the same infoset, which we denote by $I(H)$.
Figure~\ref{fig:efg-circular} illustrates the SLS for Example~\ref{ex:circular}.

A deterministic strategy $\pi$ of the questioner is given by a mapping from every infoset $I(H)$ to some question $q \in \mathcal{Q}$ to be chosen for $q_{|H|+1}$. Following $\pi$ when $s^*$ is chosen yields the following history $H_0, \dots H_T$, given by
\begin{align*}
q_{t+1} = \pi(I(H_t)), a_{t+1} = f(\pi(I(H_t)), s^*),
\end{align*}
where $T$ is the time the game ends, i.e., $|S(H_T)|=1$. The corresponding utility for the Item Chooser (resp. cost for the Questioner) at the end of the game is $u(\pi, s^*) = |H_T|$.

Let the set of all deterministic strategies be $\Pi$, $\Delta_\Pi$ the probability simplex over $\Pi$, $\Delta_\mathcal{S}$ the probability simplex over $\mathcal{S}$. Then the expected utility under random policies $x\in\Delta_\Pi$ and $y\in\Delta_\mathcal{S}$ for Questioner and Item Chooser respectively is (with some abuse of notation) denoted by
\begin{align*}
    u(x, y) = \mathbb{E}_{\pi \sim x, s^* \sim y} \left[ u(\pi, s^*)\right].
\end{align*}

The NE of this zero-sum game is given by the solution to the bilinear saddle-point problem
\begin{align}
    \min_{x \in \Delta_\Pi} \max_{y \in \Delta_\mathcal{S}}  u(x, y) = \max_{y \in \Delta_\mathcal{S}} \min_{x\in\Delta_\Pi} u( x,y )
\end{align}
where equality holds as a consequence of Von Neumann's minimax theorem \cite{v1928theorie}. The solution ($x^*, y^*$) yields optimal randomized strategies for each player, i.e., each player is best-responding to the other. In particular, $x^*$ is the \textit{distribution} over Questioner policies that is robust against the worst-case choice of $s^*$. EFGs for SLSR, WSLS and WSLSR may also be constructed by one or both of (i) restricting the actions available at each infoset to $g(S(H))$ or (ii) adjusting leaf payoffs to depend on $w$ and $s^*$.

\paragraph{Solving EFGs.} In general, two-player zero-sum EFGs with perfect recall (including SLS and its variants) can be solved in time polynomial to the size of the game tree. Some classical methods include linear programming \cite{von1996efficient} and counterfactual regret (CFR) minimization \cite{zinkevich2007regret}. Note that the definition of strategy given for the Questioner was in normal (also known as strategic or matrix) form. This leads to an exponentially sized action space, so in practice solvers usually operate (implicitly) in the strategically equivalent behavioral or \textit{sequence form} \cite{von1996efficient}. Efficient game-solving and strategy representation in EFGs is a well studied topic and beyond the scope of this paper, though we give a brief overview in the appendix. Instead, we will simply employ off-the-shelf solvers \cite{liu2024liteefgefficientpythonlibrary,lanctot2019openspiel}.

\section{Game-of-Thought and Subgame Search}

\begin{figure*}[ht]
\centering
\includegraphics[width=2.0\columnwidth]{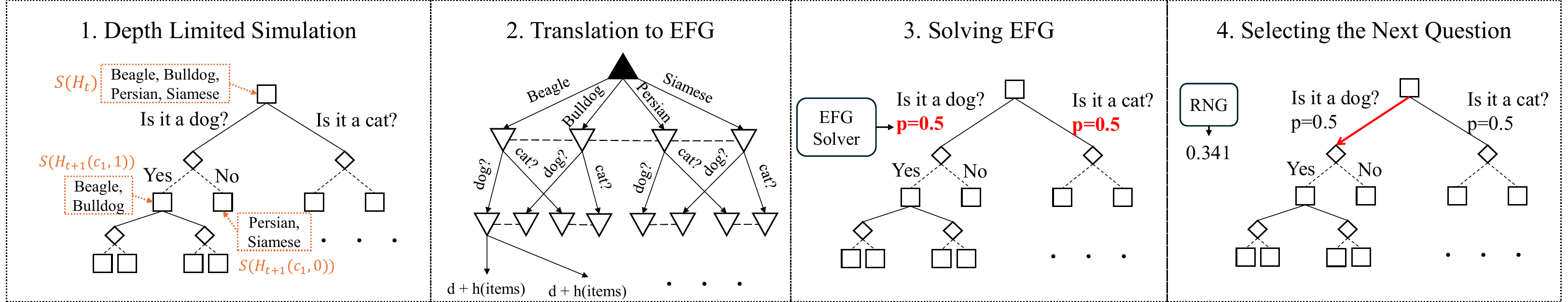}
\caption{\textbf{An overview of GoT.} To choose the question at time step $t$, we (1) perform depth limited simulation (2) Construct a subgame (3) Solve it for a local strategy which (4) informs the choice of the question. Steps 1 to 3 devise the strategy, while step 4 plays the game.}
\label{fig:method_overview}
\end{figure*}

The space of natural language questions is naturally large. Thus for practicality, we focus on the restricted variants of SLS with a small candidate questions set size $m$.
Nonetheless even under this restriction, explicitly constructing and solving the full (W)SLSR game tree is impractical for larger hypothesis spaces, as it entails querying an LLM at every infoset to propose questions and provide answers, with each call incurring multi-second latency. In our experiments, building the complete tree for a 25-item set required 5–6 hours to construct 3763 infosets. For larger datasets (approximately 100 items), this quickly becomes infeasible.

To address this challenge, GoT employs an iterative, on-demand strategy construction paradigm, computing strategies only when an infoset is reached. Inspired by the subgame search techniques used in poker bots \cite{brown2018superhuman, brown2019superhuman}, GoT constructs a subgame rooted at the current infoset and truncates it to a fixed depth $d$. The leaf nodes of this subgame are evaluated using heuristic functions. This truncated subgame is then solved to obtain a local strategy for the questioner, which informs the selection of the next question. An overview is provided in Figure~\ref{fig:method_overview}.

\paragraph{1. Depth Limited Simulation} We begin by exploring the possible future outcomes and questions for the Questioner. 
At infoset $I(H_t)$, a set of candidate questions $g(S(H_t))$ is generated using a LLM. For each candidate, we use a LLM as $f$ to identify the set of items for which the answer to $c_i$ is yes, denoted by $Y(S(H_t), c_i)$, where
\begin{align*}
 Y(S, q) =  \{ s \in S | f(q, s) = 1\} \text{ for some } S\subseteq\mathcal{S}.
\end{align*}
and the complement set denoted by $\bar{Y}(S(H_t), c_i)=S(H_t)\backslash Y(S(H_t),c_i)$. This pair of sets represents the two possible outcomes of asking the candidate question $c_i$ at step $t$, with $Y(S(H_t), c_i)$ (resp. $\bar{Y}(S(H_t), c_i)$) representing the set of items consistent with the history $H_t'(c_i, 1)$ (resp. $H_t'(c_i,0)$). The corresponding infosets $I(H_{t}'(c_i, 1))$, $I(H_{t}'(c_i, 0))$ are constructed based on this pair of outcomes. This process is repeated to construct the infosets for $d$ steps, resulting in a simulation tree as seen in step 1 of Figure~\ref{fig:method_overview}. Note that this step is similar to the question simulation and generation step of the UoT framework, and incurs a similar number of LLM calls when the same $d$ and $m$ are used.

\paragraph{2. Translation to EFG} We then construct a subgame based on the simulation tree. The tree is converted into an EFG representation of the truncated subgame, as shown in step 2 of Figure~\ref{fig:method_overview}. Note that the subgame begins with the Item Chooser ``re-choose'' a new distribution over $S(H_t)$ at each iteration even if $I(H_t)$ is not at the beginning of the game (i.e. $t\neq0)$, which seemingly accord more power to the Item Chooser.
This lets us perform \textit{safe} subgame search \cite{moravcik2016refining}, which is discussed in the appendix.
Each leaf node in the game tree $l$ is assigned a payoff of $d(l)-1 + h(l)$ for the Item Chooser (and the negative of that for the Questioner) where $d(l)$ is the depth of node $l$ in the tree representing the number of questions asked after step $t$ before reaching the node $l$, and $h$ is a heuristic function that estimates the number of questions that remain to be asked for the Questioner to identify the correct item. In our experiments for SLSR we set $h(l) :=\log_2(|S(l)|)$, which provides an optimistic estimation on the remaining turns.

\paragraph{3. Extensive Form Game Solving} To solve the subgame, we used LiteEFG's \cite{liu2024liteefgefficientpythonlibrary} implementation of CFR minimization to obtain an approximate NE.

\paragraph{4. Selecting the Next Question}
To choose the question to ask at step $t + 1$, one question $q_{t+1}$ is sampled from $g(S(H_t))$ using the the Questioner strategy from the NE. We move to the infoset $I(H_t'(q_{t+1},0))$ if $f(q_{t+1}, s^*)=0$, and $I(H_t'(q_{t+1},1))$ otherwise. 

This process is repeated until an infoset $I(H_n)$ where $|S(H_n)|=\{s^*\}$ is reached. Note that most of the computational cost of GoT is incurred in Step 1 due to LLM response latency. The remaining steps contribute negligibly.

\begin{theorem}
    GoT is safe with respect to value estimates that only depend on $S(H_t)$.
    \label{thm:safe_subgame}
\end{theorem}

\section{Experiments and Results}
We are interested in answering the following:
\textit{Does GoT improve the \textbf{worst-case} performance compared to other methods in all settings?}
We also examine how the (i) quality of the questions (ii) simulation depth $d$ affect its performance, and (iii) the trade-offs between average and worst-case.

\subsection{Experimental Setup}

\subsubsection{Datasets and Settings}
\textbf{20 Questions(20Q)}: Each dataset is a collection of distinct items. Prior studies \cite{bertolazzi-etal-2023-chatgpts, zhang-etal-2024-probing} used datasets such as \textit{Things} \cite{herbet-etal-2019-things} and \textit{Celebrities}, which we found to be too large for this study. This led to us constructing several datasets: (i) \textit{Common}: A set of 136 items built upon the dataset collected for UoT \cite{Hu2024UncertaintyOT}, which we further modified by adding more items. (ii) \textit{Breeds}: A set of 25 popular cat and dog breeds collected by us. (iii) \textit{S128}: A set of 128 items, consisting of 2 weapons, 6 scientists, 24 dishes, and 96 animals collected by us. In addition to the setting of 20Q, we also present results in two more practical domains. \textbf{Medical Diagnosis(MD)}: A doctor seeks to diagnose a patient by asking about the symptoms they are experiencing. We use \textit{DXBench(DX)}~\cite{chen2024codinterpretablemedicalagent}, and randomly select 100 out of the 461 unique diseases in the dataset as the hypothesis space. \textbf{Troubleshooting(TS)} A car mechanic attempts to help a customer diagnose the car fault they are experiencing by asking about the symptoms. We use \textit{FloDial}~\cite{raghu-etal-2021-end}, consisting of 2,738 dialogs grounded on 12 different troubleshooting flowcharts. We obtained 59 unique car faults after data preprocessing, which we used as the hypothesis space. Details are deferred to the appendix.

\subsubsection{Baselines}
We primarily compare with UoT \cite{Hu2024UncertaintyOT}. For both UoT and GoT, we set $d = m = 3$ unless otherwise stated. We additionally consider Direct Prompting (DP) where the LLM is asked to generate the next question, and Direct Choice (DC) where the LLM is asked to choose from a set of candidate questions. Each strategy is played against all possible $s^*$ in each dataset and the worst performance across all items is taken: $L_{worst} = \max_{s\in\mathcal{S}}|H^s|$, where $H^s$ is the interaction history of a method when $s$ is the chosen item.

\subsubsection{Models and Prompts}

We mainly experimented on GPT 4.1 \cite{openai2024gpt4.1} (\texttt{gpt-4.1-2025-04-14} checkpoint) and Qwen 2.5 72B Instruct \cite{Yang2024Qwen25TR}. In 20Q, we utilized two different prompts for sampling questions. 
The \textit{even} prompt, adopted from \cite{Hu2024UncertaintyOT}, explicitly instructs the LLM to ask questions that split items as evenly as possible.
When using this prompt, we noticed unnatural questions (e.g., Does this item begin with a letter from A-M? ). These questions are ill-suited for realistic settings as respondents may find it difficult to answer. This leads to the \textit{natural} prompt, which instructs LLMs to refrain from asking such questions. All results reported below use the \textit{natural} prompt. Comparison between the two prompts is deferred to the appendix. This issue was not observed in the MD and TS settings, possibly due to the explicit instruction for LLM to assume the role of a doctor/mechanic being included in the prompt. 

\subsubsection{Accounting for Randomness}
\textbf{(1)} When a LLM is used as $g$ to generate questions, the set of candidate questions $\mathcal{Q}(S(H_t))$ may differ across runs. To provide a fair comparison between similar methods, we first play the game using GoT and cache the questions in the simulation tree. The cached questions are reused for DC and UoT. \textbf{(2)} The Questioner strategy designed by GoT is nondeterministic. To obtain the performance for each $s^*$, we average over ten plays of the game using the same strategy. 

\subsection{Unweighted Variant}
As shown in \cref{tab:dls_main}, GoT consistently outperforms UoT in terms of worst-case interaction length. The prompting-based DP occasionally outperforms UoT, suggesting that while UoT can improve average-case performance, it may do so at the expense of worse worst-case performance.

\begin{table}[ht]
\centering
\begin{tabular}{c c c c c c}
    \multirow{2}{*}{\textbf{Method}} & \multicolumn{3}{c}{\textbf{20Q}} & \textbf{MD} & \textbf{TS} \\
& \textbf{Common} & \textbf{S128} & \textbf{Breeds} & \textbf{DX} & \textbf{FloDial} \\
\multicolumn{6}{c}{\cellcolor{gray!15}\textbf{GPT 4.1}} \\
GoT& \underline{10.2} & \underline{11.8} & \underline{7.4} & \underline{12.2} & \underline{7.9}\\
UoT & 11 & 13 & 9 & 13 & 9\\
DP & 13.8 & 16.2 & 7.8 & 16.8 & 12.7\\
DC & 12.9 & 14.6 & 9.3 & 16.2 & 11.6\\
\multicolumn{6}{c}{\cellcolor{gray!15}\textbf{Qwen 2.5 72B Instruct}} \\
GoT& \underline{10.0} & \underline{10.8} & \underline{6.6} & \underline{10.5} & \underline{7.5} \\
UoT & 11 & 12 & 8 & 12 & 9 \\
DP & 12.7 & 19.2 & 8.0 & 12.4 & 10.7 \\
DC & 12.7 & 17.9 & 7.8 & 13.3 & 10.5 \\
\end{tabular}
\caption{\textbf{Worst case interaction length for each method in various settings}. Best performance for each setting is \underline{underlined}.}
\label{tab:dls_main}
\end{table}

From Theorem~\ref{thm:even-split-opt-sls}, we know that if a LLM is always able to propose questions that perfectly split $S(H)$ into even halves, we would be able to achieve the optimal performance by always asking such a question.
However, we see that this is not the case: all methods fall short of the theoretical optimal performance of $\log_2(|\mathcal{S}|)$ by around 2 to 3 turns of interaction. This suggests that in practice, particularly in realistic settings such as medical diagnosis or troubleshooting, it may be difficult to phrase a sufficiently natural question that achieves an even split. In such cases, considering multiple candidate questions and adopting a randomized selection strategy, as GoT does, can hedge against outliers in which identifying a particular disease or fault requires an excessively long conversation in the worst-case.

\subsection{Weighted Variant}

We further examine the performance in the weighted variant of each setting. GoT now formulates the setting as a WSLSR game, and replaces the heuristic function used in the subgames with $h(l) = \max_{s\in S(l)}{w(s)} \cdot (d(l) + \log_2(|S(l)|))$ to reflect the change in payoff to $w(s^*)\cdot|H|$. For all methods the question sampling prompt is modified to include the item weights and the payoff calculation so as to inform the LLM.
For MD and TS, we use GPT-5.2 Thinking (with extended reasoning) to annotate each item with an integer weight from 1 to 10 reflecting the severity and urgency of the corresponding disease or fault, with 1 being the least severe and 10 being the most. Since an analogous notion is not available in 20Q, we instead sample item weights for this setting from a lognormal distribution with parameters $\mu=0$ and $\sigma=1$. A visualization of the distribution of items weights is provided in the appendix.

\begin{table}[ht]
\centering
\begin{tabular}{c c c c c}
\multirow{2}{*}{\textbf{Method}} & \multicolumn{2}{c}{\textbf{20Q}} & \textbf{MD} & \textbf{TS} \\
& \textbf{Common} & \textbf{Breeds} & \textbf{DX} & \textbf{FloDial} \\
\multicolumn{5}{c}{\cellcolor{gray!15}\textbf{GPT 4.1}} \\
GoT & \underline{152.1} & \underline{23.2} & \underline{78.3} & \underline{61.4} \\
UoT & 227.4 & 32.1 & 110.0 & 81.0 \\
DP & 224.0 & 36.9 & 116.0 & 90.1 \\
DC & 199.2 & 41.1 & 126.4 & 97.4 \\
\multicolumn{5}{c}{\cellcolor{gray!15}\textbf{Qwen 2.5 72B Instruct}} \\
GoT & \underline{151.9} & \underline{32.3} & \underline{73.6} & \underline{62.3} \\
UoT & 228.5 & 47.0 & 99.0 & 74.0 \\
DP & 235.7 & 47.8 & 114.3 & 80.0 \\
DC & 225.1 & 45.7 & 86.1 & 113.6 \\
\end{tabular}
\caption{\textbf{Worst case performance of various methods in the weighted variant, as measured by $\max_{s\in\mathcal{S}}(w(s)\cdot|H_s|)$}. Best performance for each setting is \underline{underlined}.}
\label{tab:wslsr_main}
\end{table}

The results can be seen in \cref{tab:wslsr_main}. GoT outperforms UoT in all settings, with improvement ranging from 15\% to 40\%. Another observation is that UoT's improvement over pure prompting based methods such as DP and DC is less obvious in the weighted variant. This is likely due to the UoT's approach of maximizing information gain without considering the item weights, which may lead to situations where questions yielding higher information gain (under the assumption of a uniform distribution) are preferred over those that more efficiently isolate heavily weighted items, leading to a less optimal strategy in WSLSR.

\begin{figure}[ht]
\centering
\includegraphics[width=1.0\columnwidth]{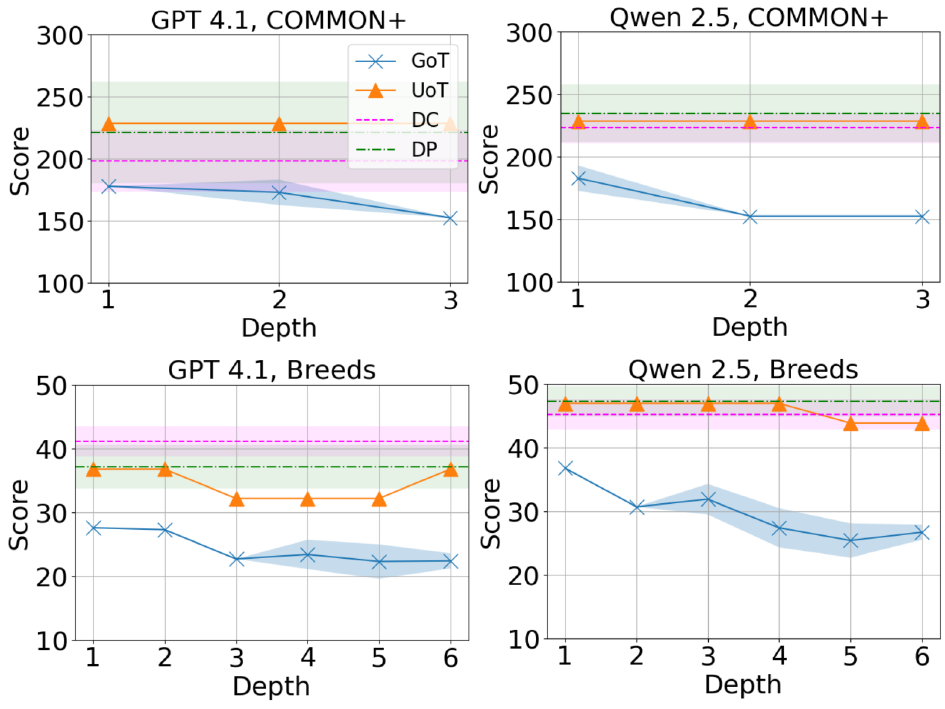}
\caption{\textbf{Worst-case performance of various methods in Weighted Variant on Common and Breeds.} The x-axis is $d$, the y-axis is the payoff for the Item Chooser $w_i \cdot|H|$.}
\label{fig:WSLSR}
\end{figure}

We further study the effect of increasing the simulation depth $d$ for both UoT and GoT; the results are shown in \cref{fig:WSLSR}. Overall, GoT improves as $d$ increases, but the gains plateau as it approaches what we hypothesize to be the optimal strategy for the fixed set of LLM-sampled questions. In contrast, UoT’s worst-case performance shows little to no improvement even with deeper lookahead, indicating that optimizing for information gain does not necessarily translate into better worst-case guarantees.

\begin{figure}[ht]
\centering
\includegraphics[width=1\columnwidth]{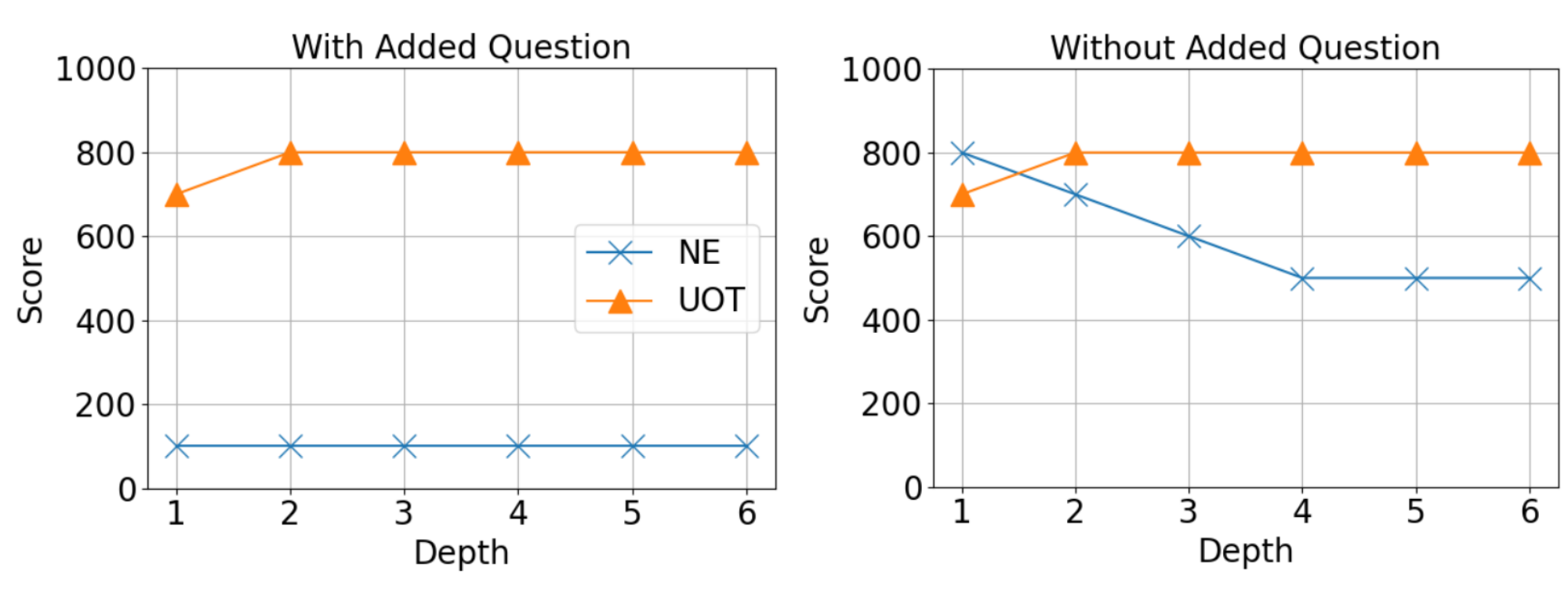}
\caption{\textbf{Performance in the weighted variant with artificially skewed weights on Breeds.} Graph on the left shows the performance when the additional question is included.}
\label{fig:skewed_weights_forced_question}
\end{figure}

We conducted a brief ablation to examine the impact of question quality in WSLSR. Specifically, we created an artificially skewed weight distribution on the Breeds dataset by assigning weight 1 to all but one item and weight 100 to the remaining item. Under this construction, the optimal first query is to test whether the highest-weight item is the target. We verify this by manually injecting the corresponding question into the initial candidate set. As shown in Figure~\ref{fig:skewed_weights_forced_question}, GoT consistently reaches this optimal strategy. In contrast, when this question is not included, performance degrades substantially, indicating that GoT’s effectiveness can be sensitive to the quality of the candidate questions.

\subsection{Average Performance of Unweighted Variant}

While we primarily examine the worst-case performance, here we also present an analysis of the average conversation length on the DX dataset under a multinomial prior $P^o$, calculated as $L_{avg}=\sum_{s\in \mathcal{S}} P^o(s)\cdot |H^s|$. This represents a method’s expected conversation length when the underlying disease distribution follows $P^o$. We obtain 50 samples of $P^o$ by drawing from a Dirichlet distribution $Dir(kC)$ where $C \in \mathbb{Z}_{100}$ is the occurrence count of each disease in the dataset, and $k\in \mathbb{R}_+$. $L_{avg}$ for each method is plotted in \cref{fig:average_case_worst_case} alongside $L_{worst}$ to illustrate the trade-off between average and worst-case performance. We also consider the setting where the item is not chosen adversarially, but rather follows a publicly known $P^o$, which allows the Questioner to specifically optimize for that prior. We refer to this strategy as the best response (BR) and it serves as a theoretical lower bound on performance. Details are in the appendix.

\begin{figure}[ht]
\centering
\includegraphics[width=1\columnwidth]{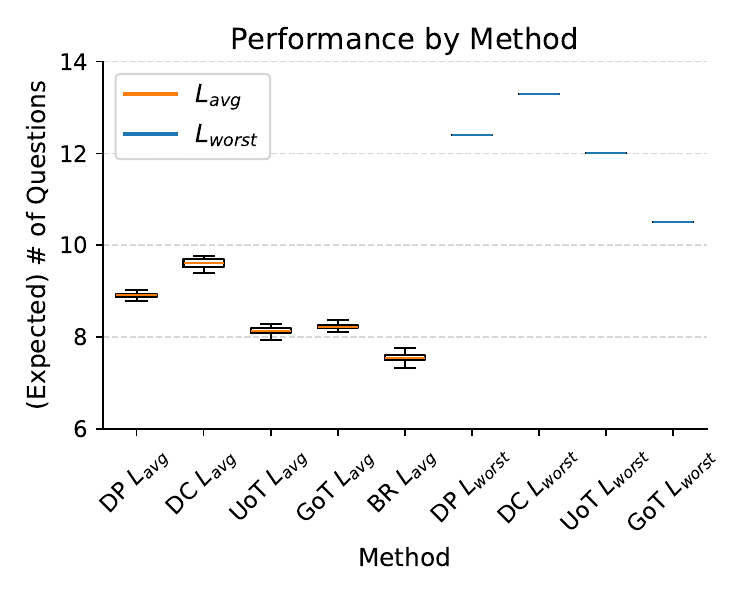}
\caption{\textbf{Average and Worst-Case Performance on DX.} The model used is Qwen 2.5 72B Instruct, with $k$ set to 1. Note that the variance in the average performance is due to different samples of $P^o$, not the randomness in the methods themselves.}
\label{fig:average_case_worst_case}
\end{figure}

From \cref{fig:average_case_worst_case}, we highlight two key observations. (1) UoT and GoT have similar average-case performance: their gap is small relative to their gains over DP/DC and to BR’s improvement over both. (2) GoT’s worst-case advantage over UoT is substantial: comparable in magnitude to UoT’s average-case gain over DP, the metric UoT explicitly targets. Overall, GoT delivers meaningful worst-case improvements with little loss in average-case performance relative to UoT.

We further evaluate $L_{avg}$ for UoT and GoT under two families of multinomial priors, chosen to be unfavorable to each method respectively. Specifically for each method, we consider $X_{method}\sim Dir(k\alpha)$, where $\alpha\in\mathbb{Z}_{100}$ satisfies $\forall i\neq0, \alpha_i=1$. We set $\alpha_0$ to concentrate the distribution around a single item. In particular, we let index $0$ correspond to the item that yields the worst case conversation length for that method. We draw 100 samples from each configuration and compare the $L_{avg}$ for UoT and GoT. 

\begin{figure}[ht]
\centering
\includegraphics[width=1\columnwidth]{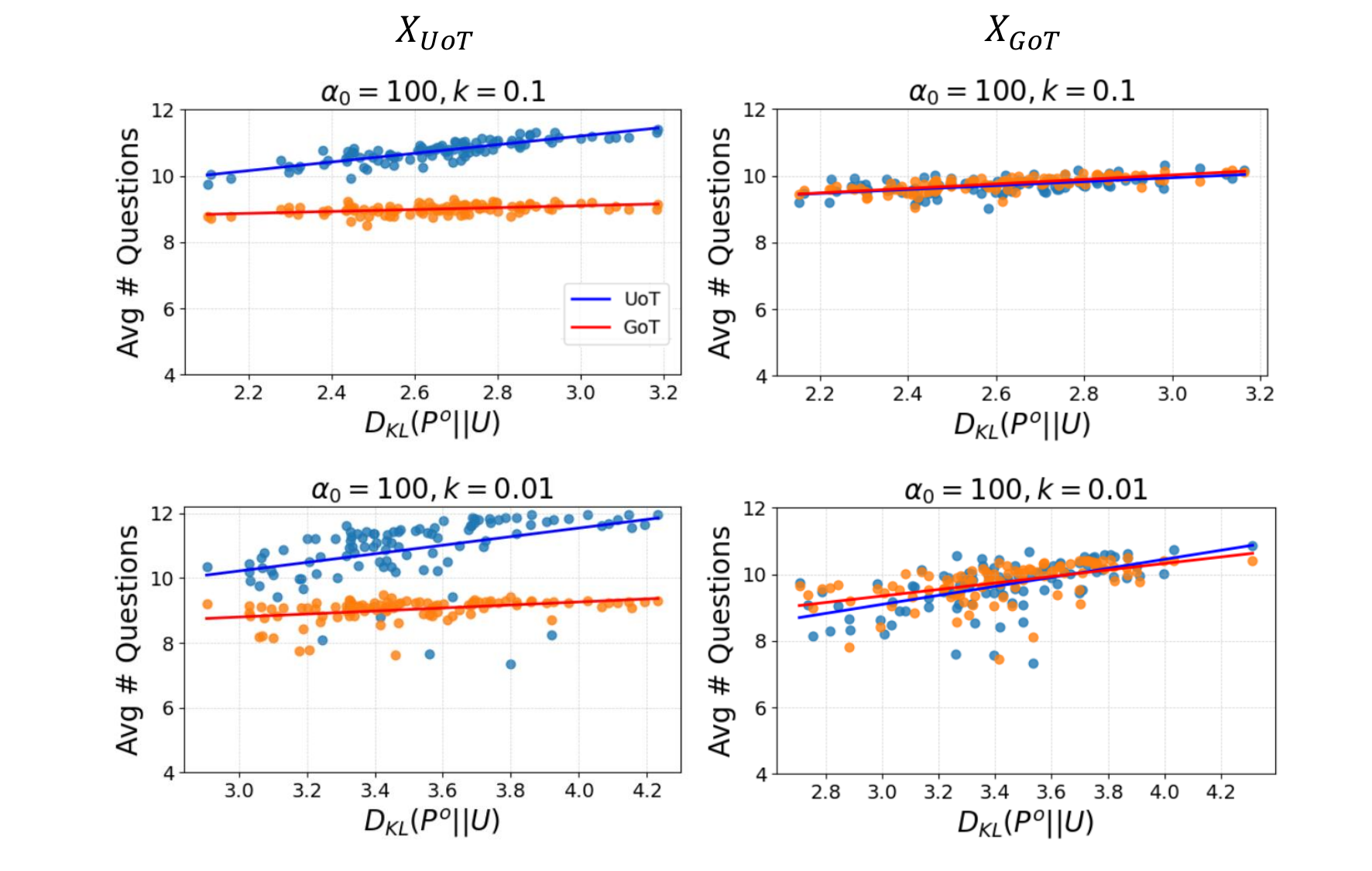}
\caption{\textbf{Average Performance in various family of distributions} Left column corresponds to priors unfavorable to UoT, while right column corresponds to priors unfavorable to GoT. X-axis is the KL divergence of the prior from the uniform distribution.}
\label{fig:average_case_scatter}
\end{figure}

As seen in \cref{fig:average_case_scatter}, for priors sampled from $X_{UoT}$, the two methods diverge markedly, with UoT's performance degrading much more than GoT. In contrast, for priors sampled from $X_{GoT}$, there is no substantial gap in average performance. Together, these results suggest that GoT is more robust to unfavorable priors than UoT.

\section{Limitations and Future Work}

We limited the candidate questions to have strictly binary answers. How this can be extended to include questions with open ended responses is left as future work.  

\section{Conclusion}

We formalized the definition of SLS and its variants, and proposed GoT, a simple yet effective approach to play SLSR. We show that GoT is able to improve the worst-case performance as compared to prior methods. Our approach of optimizing for the worst-case performance can often be more valuable than to optimize for the average performance. 

\section{Impact Statement}

This paper presents work whose goal is to advance the field of machine learning. There are many potential societal consequences of our work, none of which we feel must be specifically highlighted here.

\bibliography{refs}
\bibliographystyle{icml2026}

\newpage
\appendix
\section{Strategy Representation}
There are several ways to represent strategies (up to payoff equivalence) in imperfect information games with perfect recall. For computational reasons, the \textit{sequence form} is often used . We briefly describe the other representations.
\begin{enumerate}
    \item Normal/Strategic Form. The was what was presented in the paper. The set of deterministic strategies is the cartesian product of actions at each infoset $\Pi = \prod_{I \in \mathcal{I}} \mathcal{A}(I)$. Thus, every policy $\pi \in \Pi$ tells us exactly what action to take at every information set that could be encountered. Randomized strategies are distributions over deterministic policies, i.e., $\Delta_\Pi$. The size of $\Pi$ is exponential in the number of infosets. 
    \item Reduced Normal Form. The reduced normal form prunes the number of normal form strategies by grouping together strategically equivalent ones. See any textbook in game theory, or the work by \citet{von1996efficient} for a more complete discussion.
    Let $I$ be an infoset and $\alpha, \alpha' \in \mathcal{A}(I)$ be distinct actions in $I$. Let $I'$ be an infoset that $\alpha'$ precedes. Then, normal form strategies that choose $\alpha$ in $I$ have payoffs (regardless of opponent strategy) is independent of what action is taken in $I'$, since to enter $I'$ at all would require choosing $\alpha'$. Thus, two strategies $\pi, \pi'$ that choose $\alpha$ in $I$ but differ in actions taken in $I'$ are strategically equivalent. The minimal set of strategies that are obtained by forming such equivalence classes forms the set of reduced-normal form strategies, which can be significantly smaller than normal form ones. As before, strategies can be randomized in reduced normal form as well. Unfortunately, the size of reduced normal form strategies can still be exponential in the number of infosets.
    \item Behavioral Form. The behavioral strategies involve placing a \textit{distribution} of actions (possibly not deterministic) at each information set, thus the size is linear in the total number of actions summed over all infosets, which in turn is no larger than the size of the game tree. It can be shown using \textit{Kuhn's theorem} that under perfect recall, the set of behavioral strategies are strategically equivalent to (reduced) normal form strategies. Unfortunately, while behavioral strategies are intuitive and compact, optimizing in behavioral form is difficult as payoffs are non-linear in the strategy representation. 
    \item Sequence Form. The sequence form strategy alleviates the problems assigning probabilities to sequences $\sigma$ (essentially actions when the game has perfect recall) of taking a particular sequence in isolation from chance and other players. At each infoset $I$, the sequence form is essentially identical to the behavioral form (a probability simplex) except that they are normalized by the probabilities given by the infoset's parent sequence $\sigma(I)$, which is the last action (sequence) taken before reaching $I$ --- this is guaranteed to be unique because of perfect recall. Thus, the sequence form is the same size as behavioral strategies (except for an extra ``empty sequence'' $\phi$ set to $1.0$ that is the parent of all initial infosets. The sequence form strategy space is sometimes known as the treeplex
    \begin{align*}
        \mathcal{X} = \left\{ 
        x \in \mathbb{R}_+^N \bigg| 
        x[\phi] = 1;
        \sum_{\sigma = Ia} x[\sigma] = x[\sigma(I)] \ \forall \ I \in \mathcal{I}
        \right\}
    \end{align*}
    where $N = 1 + \sum_{I \in \mathcal{I}} |I|$ and $\mathcal{I}$ are the set of infosets, $Ia$ refers to the sequences ending with action $i$ starting at infoset $I$, and $\phi$ is the empty sequence. The vertices of $\Pi$ are correspond to (reduced) normal form strategies.
    using the sequence form, we can write the Nash equilibrium of a zero-sum game as the bilinear saddle point problem
    \begin{align*}
        \min_{x \in \mathcal{X}} \max_{y \in \mathcal{Y}} x^T M y
    \end{align*}
    where $M$ is the sequence form payoff matrix. Observe that the objective is linear in both $x$ or $y$. This saddle point can be found efficiently using a variety of methods, including first order methods such as the counterfactual regret minimization \cite{zinkevich2007regret}, linear programming or other first order methods (e.g., mirror prox). The library which we use \cite{liu2024liteefgefficientpythonlibrary} uses counterfactual regret minimization.

\end{enumerate}

\section{Detailed Derivations}
\subsection{Proof of Theorem~\ref{thm:np-complete-br-sls}}
This problem is equivalent to the following set-intersection problem.

\paragraph{Set intersection problem.} Let $\mathcal{S}$ be a finite set and $\mathcal{K} \subseteq 2^\mathcal{S}$ a set of sets 
Given some fixed $s \in \mathcal{S}$, we want to find if there is a set $\mathcal{J} \subseteq \mathcal{K}$ such that $|\mathcal{J}| \leq k$ and 
$\bigcap_{J \in \mathcal{J}} J = \{ s^* \}$, i.e., can we find a small subset of $\mathcal{J} \subseteq \mathcal{K}$ such that its intersection contains exactly $s^*$. 
One can see that $\mathcal{K}$ contains the set of questions that can be asked about $s^*$. We can safely assume that $\forall K \in \mathcal{K}, s^* \in K$. 

\paragraph{Set cover.} Let $\mathcal{A}$ be an arbitrary set and $\mathcal{B} \subseteq 2^\mathcal{A}$. The decision problem is whether one can find some set $B^* \subseteq \mathcal{B}$ of size $k$ or smaller such that $\bigcup B^* = \mathcal{A}$.

\paragraph{NP-completeness.}
We show that the set cover problem reduces to the set intersection problem. 

\begin{lemma}
    Set cover problem reduces to the Set intersection problem.
\end{lemma}
\begin{proof}
    we define $\mathcal{S} = \mathcal{A}\sqcup{s^*}$ for some dummy element $s^*$ and $\mathcal{K}$ to contain the subsets that are complements of elements of $\mathcal{B}$, with $s^*$ added, $\mathcal{K}=\{K_i=\overline{B_i}\sqcup\{s^*\}|B_i\in\mathcal{B}\}$. Take a solution $K^*\subseteq\mathcal{K}$ to this set intersection problem and the corresponding subset $B^*=\{B_i|K_i\in K^*\}$.  Since $\bigcap_{i:K_i \in K^*}K_i=\{s^*\}$, then $\bigcup_{i:K_i \in K^*}\overline{K_i}=\mathcal{S}\backslash\{s^*\}$. Thus $\bigcup_{i:B_i \in B^*}B_i=\bigcup_{i:K_i \in K^*}\overline{K_i}=\mathcal{S}\backslash\{s^*\}=\mathcal{A}$.
\end{proof}

For completeness, we also show the other direction of reduction. 

\begin{lemma}
    Set intersection problem reduces to the Set cover problem.
\end{lemma}
\begin{proof}
    We define $\mathcal{A}=\mathcal{S} \setminus \left\{s^*\right\}$ and $\mathcal{B}$ contains subsets that are the complements of the elements of $\mathcal{K}$, $\mathcal{B} = \left\{ B_i = \overline{K_i} \mid K_i \in \mathcal{K} \right\}$. 
    Take a set cover $B^* \subseteq \mathcal{B}$ of $\mathcal{A}$ and its corresponding subset $K^*=\{K_i | B_i\in B^*\}$ of $\mathcal{K}$. Since $B^*$ covers $\mathcal{A}$, i.e. $\bigcup B^*=\mathcal{A}$, we have $\bigcap_{B_i \in B^*} \overline{B_i}=\phi$, and thus $ K^*$: $\bigcap_{i: K_i\in K^*} K_i =\left\{ s^* \right\} \cup \bigcap_{B_i \in B^*} \overline{B_i}=\left\{ s^* \right\}$.
\end{proof}

Finally, any solution to the set intersection problem can also be checked in polynomial time. 

\begin{remark}
    A very similar problem known as \textit{teaching dimension} was studied by \citet{goldman1995complexity}, and is in turn closely related to other related concepts in learning theory such as the VC-dimension.
\end{remark}

\subsection{Proof of Theorem~\ref{thm:even-split-opt-sls}}
It is clear that the even-split strategy uses exactly $k$ questions. Furthermore, every question yields exactly one bit of information, so at least $k$ questions are needed. This shows that the even-split strategy is a NE (that this is a minimax solution). To show that the uniform strategy is a maximin solution, simply apply symmetry. Let $y^* \in \Delta_{n}$ be some (potentially non-uniform) NE for the answerer. We know that the set of maximin solutions for a zero-sum matrix game forms a non-empty convex set. Take all permutations of $y^*$. By symmetry, they must all be NE as well. Let the average over all permutations be be $\overline{y^*}$, again by symmetry this is equal to the uniform distribution. But by the aforementioned convexity of the set of maximin solutions this is also a Nash equilibrium for the Answerer.

\subsection{Discussion of Theorem~\ref{thm:safe_subgame}}

Step 2 of GoT augment the subgame to allow the Item Chooser to ``re-choose'' a new distribution over $s^*$. At first glance, doing this rather than sticking to the original distribution chosen by the Item Chooser at the start of the game seems to give them significantly more power than in the original SLSR game, since it would be able to ``switch items'' depending on the questions asked in actual play.. 
In fact, this is a simple implementation of \textit{maxmargin resolving}, a crucial step in performing \textit{safe} subgame search \cite{moravcik2016refining} in general zero-sum EFGs with imperfect information. 
In these games, the optimal strategy of a subgame may depend on other unreached branches of the original game and as such, we cannot solve for a strategy for this subgame in isolation.
Performing this safe subgame resolving provides a performance guarantee for the Questioner's subgame strategy with respect to some blueprint strategy or state value estimate in the original game. 
Conversely if we were to not do this, the Questioner's subgame strategy will enjoy no such guarantees and may perform arbitrarily poorly.
In practice, we see an improvement over unsafe variants of subgame search during our preliminary investigation. 

We give a very brief overview of subgame search and maxmargin resolving. The finer details and mathematical formulation is omitted. The main goal is to establish that GoT is implemented in the same way that maxmargin resolving is, which in turn implies safety.
\paragraph{Safe subgame search}
Subgame search (also known as subgame resolving or continual resolving) is a class of methods used to perform depth-limited search in a principled manner in imperfect information extensive form games. 
The idea behind subgame search is to play a \textit{blueprint} strategy until the player enters a \textit{subgame} (an imperfect information subgame is a forest
of trees, closed under both the descendant relation and membership within augmented information sets for any player, \cite{burch2014solving}). The blueprint strategy is \textit{typically} the solution to a coarse, abstract version of the full game. Once inside the subgame, the player \textit{resolves} for a better solution for the subgame that was reached (and that subgame only). This is analogous to the perfect information case where one only performs search (or refinement) of a strategy in the state that was reached in actual play, since solving the original full game is intractable. 

A safe subgame solving algorithm is one that is guaranteed to perform no worse than the blueprint strategy. Naive subgame solving attempts to solve the subgame by constructing a gadget game starting with a chance node which leads to all initial states in the subgame based on the probabilities (under the blueprint and the opponent strategy whens solving the blueprint) of reaching them. This however, neglects the fact that resolving just the subgame that was reached could entice the opponent to change their action, rendering these chance probabilities incorrect, leading to a worse performance than the blueprint.

\paragraph{Maxmargin resolving}
Maxmargin resolving \cite{moravcik2016refining} is one such approach to perform subgame solving in a more principled fashion. The idea is to augment the gadget game such that the \textit{opponent} first chooses which infoset (of the opponent) it would like to begin from, after which, a chance node enforces the probabilities of reaching each of those states (belonging to the opponent's infoset) is, based on the blueprint strategy of the main player. Letting the opponent choose which infoset (note that we may have to add dummy infosets if the opponent does not move at the beginning of the subgame) they want to begin the subgame with ensures that the main player optimizes for the minimum \textit{margin}. Here, the margin for each of the opponent's infoset is the difference between the value in the infoset under the blueprint strategy versus the refined strategy. If the main player optimizes the value of a particular head infoset (of the opponent) too much at the expense of performing poorly at other infosets, then the opponent would choose those infosets instead, resulting in unsafe resolving. By allowing the opponent to choose initial infosets, Maxmargin avoids this explicitly by ensuring that all of the infosets are equally improved (for the main player) after refinement.

In practice, we do not actually know the values of initial states of the subgame under the best-response of the opponent is (without expanding the entire subgame to its leaves) and have to resort to approximate value functions. Note that it is not immediately clear what constitutes a good value function, since the value of a state will depend on play from both players and is complicated by imperfect information. See \citet{kovarik2021value} for a more thorough discussion. 

\paragraph{GoT as maxmargin resolving}
First, observe that SLSR is a relatively simple EFG in that the Answerer only moves once at the beginning, after which there are no further interactions from it. In fact, we could allow the Answerer to have perfect information. Furthermore, the proper subgames (i.e., not the full game) are simply the histories $H$, or equivalently, the subgame beginning at $I(H)$, which comprise all questions asked but not what $s^*$ was chosen. Note that each proper subgame only has the Questioner taking actions. Since the Answerer has perfect information the ``head infosets'' of a subgame $I(H)$ for the Answerer simply corresponds to the initial states of that history $H$. Note that we don't actually need to explicitly construct these head infosets: since the Answerer has full information, these are ``dummy'' infosets have essentially one action only, which lead to the corresponding state in the lead infoset of the Questioner.
This shows that the structure of the gadget game of maxmargin is exactly the same as what we propose. An example of this construction is shown in Figure~\ref{fig:efg-subgame}.

In maxmargin, the payoffs under each head infoset of the subgame is shifted by the payoffs under the blueprint (or function approximation). This ensures that the \textit{margin} is optimized as opposed to absolute payoffs. Thankfully, no such shift is necessary in our case because the approximated values at each state is $\log_2(|S(H)|)$, i.e., the log of the number of items remaining. This value is the same for \textit{every} state in the infoset $I(H)$, thus the value of the shift is the same over the entire subgame, which, as far as solving the subgame goes has no bearing on the refined strategy. 

This same argument also extends to our choice of value function for WSLSR, $h(l) = \max_{s\in S(l)}{w(s)} \cdot (d(l) + \log_2(|S(l)|))$, which depends only on the set of items remaining, $S(H)$. However, if our value function was chosen to also depend on the Answerer's precise choice of $s^*$, then such an argument would no longer hold and we would have to perform these shifts in payoffs accordingly.

\begin{remark}
    In theory, maxmargin resolving only performs resolving upon entering a subgame (e.g., after $d=3$ questions are asked). In practice, GoT performs resolving at every infoset that is encountered during actual play. 
\end{remark}

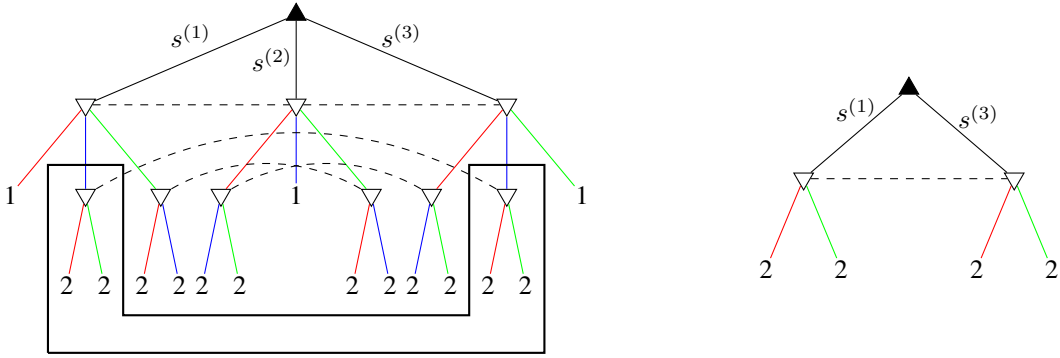
\begin{figure*}[ht]
\begin{minipage}[c]{0.47\linewidth}
\centering 
\begin{tikzpicture}[
  level distance=12mm,
  every node/.style={circle,draw,minimum size=2mm,inner sep=0pt},
  level 1/.style={sibling distance=28mm},
  level 2/.style={sibling distance=10mm},
  level 3/.style={sibling distance=5mm},
  invertednode/.style={
    regular polygon,
    regular polygon sides=3,
    draw,
    minimum size=3mm,
    inner sep=0pt,
    rotate=180,
  },
]
\usetikzlibrary{trees,shapes.geometric}

\node[regular polygon, regular polygon sides=3, draw, fill=black, minimum size=3mm, inner sep=0pt] {}
  child { node[invertednode] (v1){}
    child[edge from parent/.style={draw=red}] { node[draw=none] {1}
    }
    child[edge from parent/.style={draw=blue}] { node[invertednode] (v13a) {}
      child[edge from parent/.style={draw=red}] { node[draw=none] {2} }
      child[edge from parent/.style={draw=green}] { node[draw=none] {2} }
    }
    child[edge from parent/.style={draw=green}] { node[invertednode] (v12a){}
      child[edge from parent/.style={draw=red}] { node[draw=none] {2} }
      child[edge from parent/.style={draw=blue}] { node[draw=none] {2} }
    }
    edge from parent node[draw=none,pos=0.5,above] {$s^{(1)}$}
  }
  child { node[invertednode] (v2){}
    child[edge from parent/.style={draw=red}] { node[invertednode] (v23a){}
      child[edge from parent/.style={draw=blue}] { node[draw=none] {2} }
      child[edge from parent/.style={draw=green}] { node[draw=none] {2} }
    }
    child[edge from parent/.style={draw=blue}] { node[draw=none] {1}
    }
    child[edge from parent/.style={draw=green}] { node[invertednode] (v12b){}
      child[edge from parent/.style={draw=red}] { node[draw=none] {2} }
      child[edge from parent/.style={draw=blue}] { node[draw=none] {2} }
    }
    edge from parent node[draw=none,pos=0.5,left] {$s^{(2)}$}
  }
  child { node[invertednode] (v3){}
    child[edge from parent/.style={draw=red}] { node[invertednode] (v23b){}
      child[edge from parent/.style={draw=blue}] { node[draw=none] {2} }
      child[edge from parent/.style={draw=green}] { node[draw=none] {2} }
    }
    child[edge from parent/.style={draw=blue}] { node[invertednode] (v13b) {}
      child[edge from parent/.style={draw=red}] { node[draw=none] {2} }
      child[edge from parent/.style={draw=green}] { node[draw=none] {2} }
    }
    child[edge from parent/.style={draw=green}] { node[draw=none] {1}
    }
    edge from parent node[draw=none,pos=0.5,above] {$s^{(3)}$}
  };
\draw[dashed] (v1) -- (v2);
\draw[dashed] (v2) -- (v3);
\draw[dashed] (v23a) to[out=30,in=150] (v23b);
\draw[dashed] (v13a) to[out=30,in=150] (v13b);
\draw[dashed] (v12a) to[out=30,in=150] (v12b);

\draw[thick] 
(-3.3,-4.5) -- (-3.3,-2) -- (-2.3,-2) -- (-2.3,-4.0) -- (2.3,-4.0) -- (2.3,-2) -- (3.3,-2) -- (3.3,-4.5) -- (-3.3,-4.5);
\end{tikzpicture}
\end{minipage}
\begin{minipage}[c]{0.47\linewidth}
\centering
\begin{tikzpicture}[
  level distance=12mm,
  every node/.style={circle,draw,minimum size=2mm,inner sep=0pt},
  level 1/.style={sibling distance=28mm},
  level 2/.style={sibling distance=10mm},
  level 3/.style={sibling distance=5mm},
  invertednode/.style={
    regular polygon,
    regular polygon sides=3,
    draw,
    minimum size=3mm,
    inner sep=0pt,
    rotate=180,
  },
]
\usetikzlibrary{trees,shapes.geometric}

\node[regular polygon, regular polygon sides=3, draw, fill=black, minimum size=3mm, inner sep=0pt] {}
  child { node[invertednode] (v1){}
    child[edge from parent/.style={draw=red}] { node[draw=none] {2} }
    child[edge from parent/.style={draw=green}] { node[draw=none] {2} }
        edge from parent node[draw=none,pos=0.5,above] {$s^{(1)}$}
  }
  child { node[invertednode] (v3){}
      child[edge from parent/.style={draw=red}] { node[draw=none] {2} }
      child[edge from parent/.style={draw=green}] { node[draw=none] {2} }
    edge from parent node[draw=none,pos=0.5,above right] {$s^{(3)}$}
  };
\draw[dashed] (v1) -- (v3);
\end{tikzpicture}
\end{minipage}

\caption{Left: example of a subgame based on Example~\ref{ex:circular} and Figure~\ref{fig:efg-circular}. Bounded region by thick dark lines form an example of a subgame corresponding to the history of asking $q^{(2)}$ and getting an answer of $1$. Right: the gadget game used for solving the subgame shown on the left. Note that it is the Answerer who takes the first action, not chance (which would be the case for naive subgame solving), and the actions only include the items that are consistent with $H$, in this case $s^{(1)}$ and $s^{(3)}$.}
\label{fig:efg-subgame}
\end{figure*}

\section{Implementation Details}

\subsection{General Implementation}
\paragraph{Method Hyperparameters} Since UoT shares a similar simulation procedure with GoT, we use the same values for the simulation depth $d$ and the number of candidate questions $m$ for both methods within each experiment, although these values may vary across different experiments. For DC, we similarly provide a set of $m$ candidate questions which it may choose from.

\paragraph{Question Generation When There Are Two Items Left} In instances during a (W)SLSR game where the current history $H$ satisfies $|S(H)| = 2$, we deterministically construct two candidate questions of the form “Is x the correct item?” for each remaining item, instead of sampling questions from the LLM. While not strictly required, it is implemented primarily for efficiency. Under \cref{ass:can_always_progress}, any question that satisfies the assumption will successfully distinguish between the two remaining items, thereby terminating the game. This eliminates the need to query an LLM at such states, thereby reducing the latency associated with constructing the simulation tree.

\paragraph{Question Caching and Reuse} Since we construct the SLSR game on demand in our experiments, we need to ensure the consistency between the game instance played by each of the methods. During each execution of GoT, the set of candidate questions $g(S(H))$ generated by the LLM for any explored history $H$ is cached for reuse. In the case of UoT and DC, if the current history has previously been explored by GoT, the corresponding cached set of candidate questions is reused in place of sampling new questions from the LLM. This is done to ensure a fair comparison across methods, as variations in the quality of candidate questions can substantially impact performance. 

\paragraph{Handling erroneous LLM outputs} For ease of parsing, we require the LLMs to produce outputs in JSON format. While the model occasionally fails to adhere to this format, we address such cases by retrying the generation. In our experiments, this retry mechanism proved effective, and no further issues were observed.

\subsection{GoT}
\paragraph{When Simulation Tree is Terminal}
During Simulation Step (1), after constructing the local simulation tree, we verify whether each leaf node corresponds to a terminal state, i.e., whether $|S(l)| = 1$ holds for all leaf nodes. If this condition is met, the strategy derived from the current subgame is adopted for all subsequent question selections until the end of the game, rather than being used solely for the next question followed by resolving a new subgame at the next step. This is justified by the fact that the current subgame extends to the end of the game, thereby obviating the need for iterative strategy construction in subsequent steps.

\subsection{BR}
\paragraph{Obtaining Best Response}
To obtain a best response (BR), we consider a non-adversarial setting in which the item (distribution) is drawn from a known prior $P^o$ rather than chosen adversarially. Under this assumption, the interaction can be modeled as a single-player extensive-form game with imperfect information, analogous to our SLS representation, by replacing the root node corresponding to the Item Chooser with a chance node whose outgoing edges are weighted according to $P^o$. The optimal Questioner strategy in this case is deterministic and should minimize the expected cost incurred. We obtain this strategy by adopting an iterative, on demand framework similar to GoT, replacing step 3 with a simple backward induction.
\begin{example}
    Consider a game similar to Example~\ref{ex:circular}, with $s^{(1)}, s^{(2)}, s^{(3)}$ being distributed under the prior $P^o=(0.8, 0.1, 0.1)$.
    \label{ex:pomdp}
\end{example}
In Example~\ref{ex:pomdp}, the Questioner should always ask $q^{(1)}$ first, followed by either $q^{(2)}$ or $q^{(3)}$. This yields an expected cost of $1.2$. \cref{fig:pomdp} illustrates the extensive form representation of \cref{ex:pomdp}.

\begin{figure}[ht]
\begin{tikzpicture}[
  level distance=12mm,
  every node/.style={circle,draw,minimum size=2mm,inner sep=0pt},
  level 1/.style={sibling distance=28mm},
  level 2/.style={sibling distance=10mm},
  level 3/.style={sibling distance=5mm},
  invertednode/.style={
    regular polygon,
    regular polygon sides=3,
    draw,
    minimum size=3mm,
    inner sep=0pt,
    rotate=180,
  },
]
\usetikzlibrary{trees,shapes.geometric}

\node[circle, draw, fill=black, minimum size=3mm, inner sep=0pt] {}
  child { node[invertednode] (v1){}
    child[edge from parent/.style={draw=red}] { node[draw=none] {1}
    }
    child[edge from parent/.style={draw=blue}] { node[invertednode] (v13a) {}
      child[edge from parent/.style={draw=red}] { node[draw=none] {2} }
      child[edge from parent/.style={draw=green}] { node[draw=none] {2} }
    }
    child[edge from parent/.style={draw=green}] { node[invertednode] (v12a){}
      child[edge from parent/.style={draw=red}] { node[draw=none] {2} }
      child[edge from parent/.style={draw=blue}] { node[draw=none] {2} }
    }
    edge from parent node[draw=none,pos=0.5,above] {$0.8$ }
  }
  child { node[invertednode] (v2){}
    child[edge from parent/.style={draw=red}] { node[invertednode] (v23a){}
      child[edge from parent/.style={draw=blue}] { node[draw=none] {2} }
      child[edge from parent/.style={draw=green}] { node[draw=none] {2} }
    }
    child[edge from parent/.style={draw=blue}] { node[draw=none] {1}
    }
    child[edge from parent/.style={draw=green}] { node[invertednode] (v12b){}
      child[edge from parent/.style={draw=red}] { node[draw=none] {2} }
      child[edge from parent/.style={draw=blue}] { node[draw=none] {2} }
    }
    edge from parent node[draw=none,pos=0.5,left] {$0.1$}
  }
  child { node[invertednode] (v3){}
    child[edge from parent/.style={draw=red}] { node[invertednode] (v23b){}
      child[edge from parent/.style={draw=blue}] { node[draw=none] {2} }
      child[edge from parent/.style={draw=green}] { node[draw=none] {2} }
    }
    child[edge from parent/.style={draw=blue}] { node[invertednode] (v13b) {}
      child[edge from parent/.style={draw=red}] { node[draw=none] {2} }
      child[edge from parent/.style={draw=green}] { node[draw=none] {2} }
    }
    child[edge from parent/.style={draw=green}] { node[draw=none] {1}
    }
    edge from parent node[draw=none,pos=0.5,above] {$0.1$}
  };
\draw[dashed] (v1) -- (v2);
\draw[dashed] (v2) -- (v3);
\draw[dashed] (v23a) to[out=30,in=150] (v23b);
\draw[dashed] (v13a) to[out=30,in=150] (v13b);
\draw[dashed] (v12a) to[out=30,in=150] (v12b);
\end{tikzpicture}
\caption{EFG representation of Example~\ref{ex:pomdp}. 
The root node \tikz[baseline=-0.65ex]{\node[circle ,draw,minimum size=3mm,inner sep=0pt,rotate=0,fill=black] {};} is a chance node and the other nodes 
\tikz[baseline=-1ex]{\node[regular polygon,regular polygon sides=3,draw,minimum size=3mm,inner sep=0pt,rotate=180] {};}
belong to the Questioner. The respective probabilities of each item is shown.}
\label{fig:pomdp}
\end{figure}
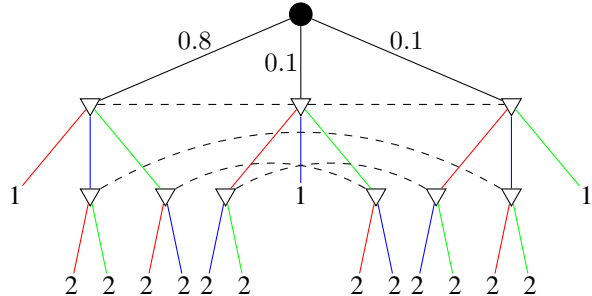

Note that in this setting, the problem is equivalent to planning under partial observability, i.e., a Partially Observable Markov Decision Process (POMDP)\cite{Spaan12pomdp}.
\section{Additional Results}

\subsection{SLSR with Synthetic Splits}
The quality of a candidate question $q$ in SLSR at some infoset $I(H)$ can be roughly measured using the ratio $|Y(S(H), q)| : |\bar{Y}(S(H), q)|$, representing how evenly $q$ is able to split $S(H)$ into two.
Given the difficulty in controlling the quality of LLM-sampled questions, we wish to study SLSR in a more controlled environment.
Instead of using a LLM for question generation, we define a question generating function $g': 2^{\mathcal{S}}\setminus\phi\times(0,1)\rightarrow2^{\mathcal{S}}$ which takes as input $S(H)$ and a split ratio $r\in(0,1)$, and outputs the set\footnote{The question itself does not matter in this case, as both UoT and GoT is language agnostic.} $Y'\subset S(H)$ (and consequently $\bar{Y}'=S(H)\setminus Y'$) at the fixed ratio $\frac{|Y'|}{|S(H)|}=r$. Two implementations of $g'$ were considered: (i) Random Splits: Sample and return $r\cdot |S(H)|$ items uniformly at random from $S(H)$ without replacement. (ii) Feature Based Splits: For each $s \in S(H)$, generate $k$ features $F_i=(f_1, ..., f_k), f_j\in[0,1]$ at the start of the game. $g'$ randomly selects one of $k$ features, and sorts $S(H)$ based on the selected feature. The top $r\cdot|S(H)|$ items are returned.

The results are shown in Table~\ref{tab:synthetic_splits_features}. 
GoT consistently achieves superior performance compared to UoT, particularly as $r$ decreases and the splits become increasingly skewed. This suggests that GoT provides the most improvement when the candidate questions are not optimal. This however does not translate well when actual questions are used due to other confounding factors not captured.

\begin{table}[ht]
\centering
\begin{tabular}{p{1.2cm} p{1.5cm} p{1.5cm} p{1.5cm}}
\multirow{2}{*}{\textbf{Method}} & \multicolumn{3}{c}{$r$} \\
& \textbf{0.4} & \textbf{0.33} & \textbf{0.25} \\
\multicolumn{4}{c}{\cellcolor{gray!15}\textbf{3 Features}} \\
GoT & 9.4 & 10.0 & 13.8 \\
UoT & 10.0 & 11.0 & 15.0 \\
\multicolumn{4}{c}{\cellcolor{gray!15}\textbf{5 Features}} \\
GoT& 9.2 & 10.2 & 13.2 \\
UoT & 10 & 11.0 & 15.0 \\
\multicolumn{4}{c}{\cellcolor{gray!15}\textbf{Random Splits}} \\
GoT& 8.8 & 9.8 & 11.6 \\
UoT & 10 & 11 & 15 \\
\end{tabular}
\caption{\textbf{Worst case interaction length on SLSR with Synthetic splits with Common}. $d$ is set to be 3.}
\label{tab:synthetic_splits_features}
\end{table}

\subsection{Full SLSR Games}

For each SLSR game, there exists a lower bound on the expected worst case number of questions required. This bound can be obtained by solving the fully specified game. As this is only feasible for smaller datasets, we constructed complete SLSR games on the smaller Breeds dataset and experimented on GoT and UoT by setting the simulation depth $d$ to game tree depth $D$. Since constructing the full game tree is required, we enforce \cref{ass:will_always_progress} to ensure that all branches eventually terminate. The results are shown in \cref{tab:full_search}, where GoT’s performance aligns with the lower bound on the performance. On the other hand, UoT’s performance showed minimal improvement even when provided with the full game tree during simulation, suggesting that achieving the lower bound requires a non-deterministic strategy.

\begin{table}[ht]
\centering
\begin{tabular}{>{\centering}p{1cm} c c}
\multirow{2}{*}{\textbf{Method}} & \multicolumn{2}{c}{\textbf{Experiment}} \\
& \textbf{Subgame ($d=3$)} &\textbf{Full Game ($d=D$)} \\
\multicolumn{3}{c}{\cellcolor{gray!15}\textbf{GPT 4.1 + Even split prompt}} \\
GoT& 6.4 & 5.84 \\
UoT & 7 & 7\\
\multicolumn{3}{c}{\cellcolor{gray!15}\textbf{GPT 4.1 + Natural prompt}} \\
GoT& 7.4 & 5.64\\
UoT & 9 & 8\\
\multicolumn{3}{c}{\cellcolor{gray!15}\textbf{Qwen 72B Instruct + Even split prompt}} \\
GoT& 6.2 & 5.44\\
UoT & 7 & 7\\
\multicolumn{3}{c}{\cellcolor{gray!15}\textbf{Qwen 72B Instruct + Natural prompt}} \\
GoT& 6.6 & 5.88\\
UoT & 8 & 8\\
\end{tabular}
\caption{\textbf{Worst case interaction length for on full SLSR Game.} Dataset used is the Breeds Dataset.}
\label{tab:full_search}
\end{table}

\subsection{Item Weights}
Figure~\ref{fig:weights_dist} presents a histogram of item weights used in the weighted variant for the Breeds and Common datasets in the 20Q setting, DX in the MD setting, and Flodial in the TS setting. For datasets in the 20Q setting, the weights were assigned randomly, as there was no clear method for determining meaningful valuations for each item given the nature of the datasets. For TS and MD, the weights are obtained by asking Chatgpt 5.2 to annotate each disease/fault with a severity score. The prompt used is shown in \cref{tab:annotate_weight}

\begin{table}[ht]
\centering
\begin{tabular}{p{7.5cm}}
\cellcolor{gray!15}\textbf{MD}\\
Here is a list of diseases. I want you to assign a severity score between 1 to 10 for each of these diseases. This severity score should depend on how life threatening the disease can be, and how urgently it needs treatment or care.\\
The list of diseases:\\
\textbf{[diseases]}\\
\cellcolor{gray!15}\textbf{TS}\\
Here is a list of car faults. I want you to assign a severity score between 1 to 10 for each of these faults. This severity score should depend on how severe the fault can be, and how urgently it needs to be repaired.\\ 
The list of car faults:\\
\textbf{[faults]}\\
\end{tabular}
\caption{\textbf{Prompt used to annotate the item weights for the MD and TS settings.}}
\label{tab:annotate_weight}
\end{table}

It is important to note that the performance of GoT in a WSLSR game may be sensitive to the underlying distribution of item weights. In particular, when the weights were sampled from a uniform distribution, we observed that the relative performance of GoT compared to UoT mirrored the results in the unweighted setting, providing a consistent but smaller improvement.

\begin{figure}[ht]
\centering
\includegraphics[width=1.0\columnwidth]{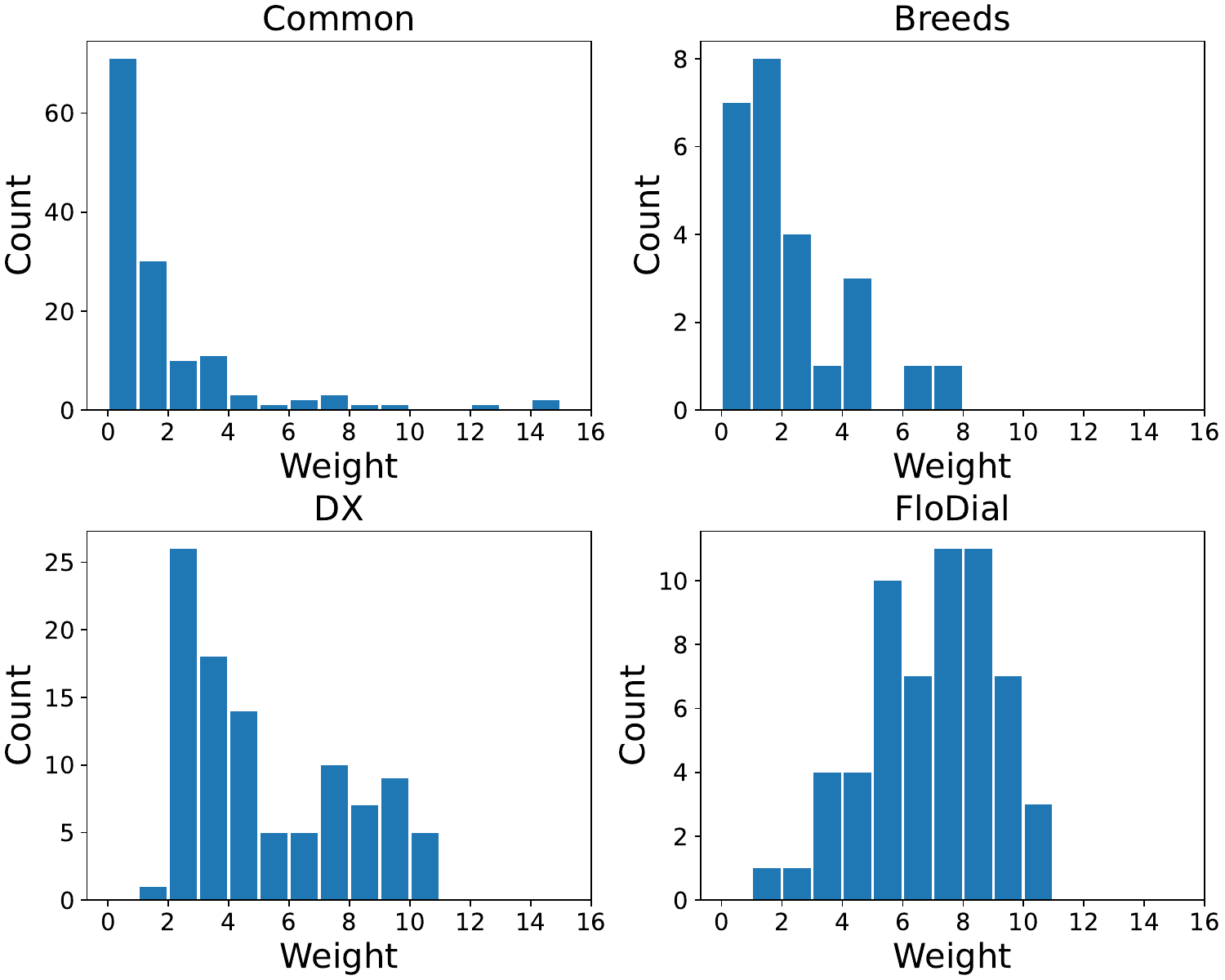}
\caption{\textbf{Distribution of item weights used in our experiments.} Weights are sampled from a lognormal distribution with parameters $\mu=0, \sigma=1$.}
\label{fig:weights_dist}
\end{figure}

\subsection{Randomization of GoT's strategies}
We quantified the randomness of GoT’s strategy using entropy, as shown in Figure~\ref{fig:strat_entropy}. In general, the resulting strategies exhibit higher entropy in the standard SLSR game compared to the WSLSR variant. This is likely due to the item weight distribution being tail-heavy, which encourages the strategy to favor candidate questions that more efficiently isolate the heavily weighted items. Nonetheless the strategies remain sufficiently random, suggesting that we do not encounter degenerate cases in which a deterministic strategy would be optimal.

\begin{figure}[ht]
\centering
\includegraphics[width=1.0\columnwidth]{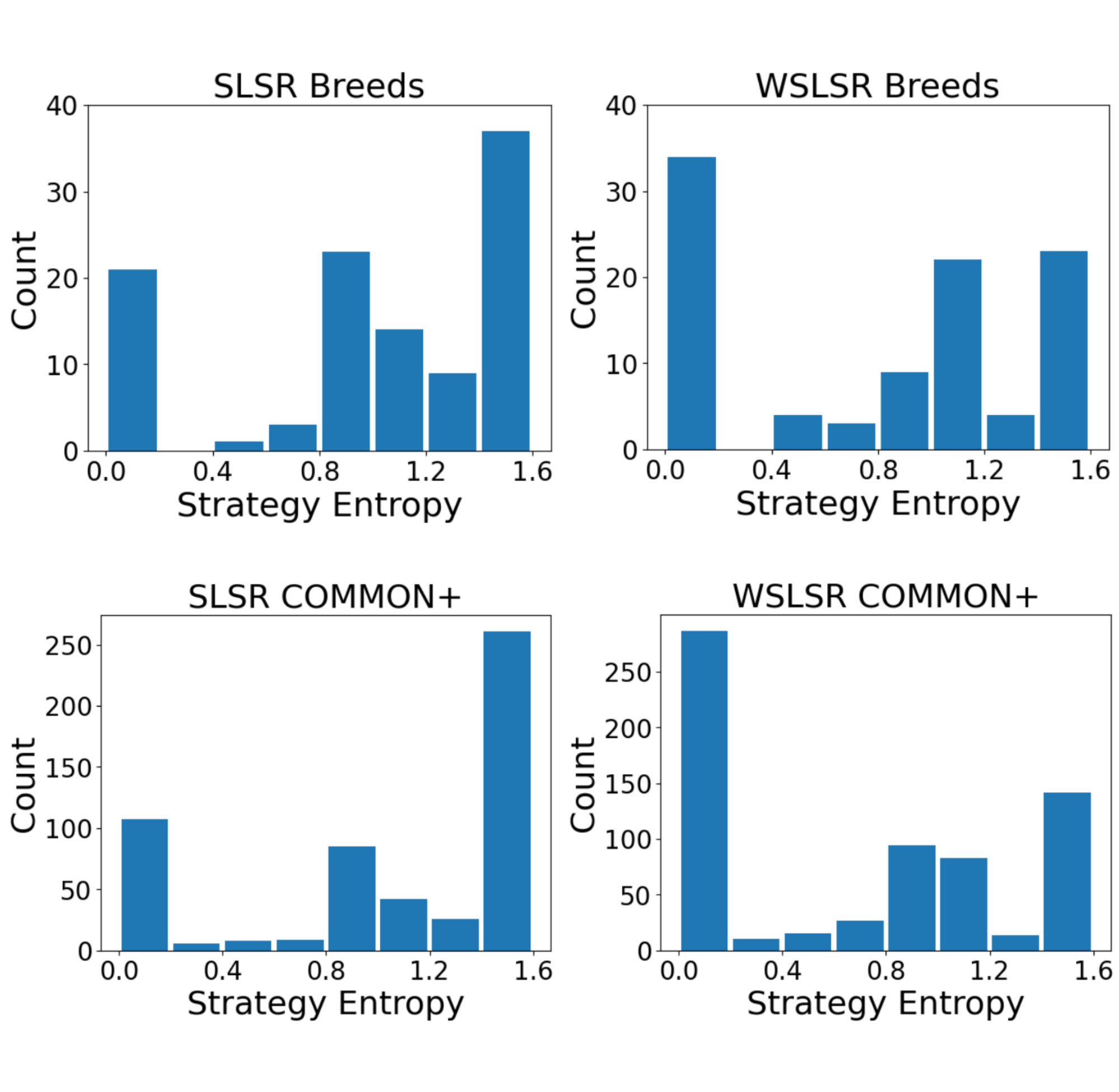}
\caption{\textbf{Entropy of the strategies obtained using GoT at each infoset.} As $m=3$, the maximum entropy of the strategy is $\log_2(3)\approx1.58$ when every candidate question is chosen with equal probability. Strategies for when there are only two items left are not included.}
\label{fig:strat_entropy}
\end{figure}

\section{Discussion on Assumptions}
All of the aforementioned assumptions are either naturally satisfied by the experimental setup or are explicitly enforced during the execution of our method.

\paragraph{\cref{ass:unique_ans}} 
In our formulation of the SLS problem, we assume that (i) the codomain of $f$ is binary, and (ii) each item $s \in \mathcal{S}$ has a unique, well-defined answer under $f$. However, this assumption may not hold in real-world scenarios due to the inherent ambiguities and interpretive variability commonly present in natural language. One source of such ambiguity arises when an item name refers to multiple distinct entities. For example, consider the question: Is the Titanic something a person can pick up with their hands? The answer would be no if referring to the actual ship, but yes if referring to a DVD of the movie Titanic.
In our experiments, to abstract away such linguistic ambiguity, we instantiate the labeling function  $f$ with an LLM and treat its output as a definitive label for each item.

\paragraph{\cref{ass:no_lying}} This naturally follows from \cref{ass:unique_ans,ass:reliable_llm} and the game structure.

\paragraph{\cref{ass:common_knowledge}} This naturally follows from the game structure.

\paragraph{\cref{ass:oracle-access}} We employed LLMs, accessed both via online APIs and locally hosted instances, to implement the function $f$. For online API-based queries, the average latency per query is approximately 1 to 5 seconds, while for locally hosted models, the latency is slightly lower, typically around 1 to 3 seconds. 

\paragraph{\cref{ass:can_always_split}} This assumption is satisfied as $\mathcal{Q}$ is the set of all possible natural language questions. In particularly it is trivially satisfied by the presence of identity questions.

\paragraph{\cref{ass:can_always_progress}} This assumption holds in nearly all cases when LLMs are used as the question generator $g$. To enforce this, we verify whether there exists at least one question from the sampled candidate set that permits progression of the game; if no such question is found, we resample the candidate questions.
For certain experiments (such as full SLSR game solving) we impose a stronger version of this assumption:

\begin{assumption}
    For every subset $S\subseteq\mathcal{S}$, for every $q\in g(S)$, there exists a pair of distinct items  $s, s'\in S$ where $f(q, s) \neq f(q, s')$
    \label{ass:will_always_progress}
\end{assumption}

\noindent This guarantees that the game tree will be finite, thereby enabling the complete construction of the game tree. To enforce this condition, we explicitly remove any candidate question that fails to satisfy the assumption from the set of candidate questions.

\paragraph{\cref{ass:reliable_llm}} This assumption is made to mitigate potential inaccuracies arising from LLM-based classifications. Although LLMs may produce erroneous outputs, handling such errors is beyond the scope of our method and is therefore not the primary focus of this work. Without this assumption, classification errors during the simulation process could, in the worst case, require a restart of the game. Specifically, if the target item $s^*$ is misclassified, it may be incorrectly eliminated from the set of possible items upon receiving a contradictory response from the answerer. Moreover, we observed that such misclassifications and consequently restarts occurred frequently when attempting to reproduce the results of \citet{Hu2024UncertaintyOT}, which drastically affected the worst case performance. Since this does not influence the Questioner’s strategy for any of the evaluated methods, we believe it does not compromise the fairness of the experimental results.
To enforce this assumption for both GoT and UoT, we depart from the approach used in \citet{Hu2024UncertaintyOT}, where a separate LLM instance was employed to simulate the Answerer. Instead, we cache the item set splits generated during the simulation step corresponding to each question and reuse them as the Answerer's response during our experiments. This strategy ensures consistency between the item classifications performed during the simulation phase and the responses observed during the actual execution of each method.

\section{Dataset Details}
All constructed datasets (along with item weights used for the weighted variant) are released alongside the accompanying code. 

\paragraph{20Q} To construct the datasets, we relied on publicly available sources. For example, the American Kennel Club for the Breeds dataset, and ChatGPT for the S128 dataset. All AI-generated datasets were subsequently verified to ensure that each item included corresponds to a real-world entity. The size of these datasets is intentionally designed to be comparable to those used in \citet{Hu2024UncertaintyOT}, since both their method (UoT) and GoT shares a similar computational complexity due to the future simulation steps. Game-solving in GoT contributes only a minor fraction to the running time, and GoT can be seen as a robust extension to UoT at little extra cost. This is also why we primarily compare our method against UoT.

\paragraph{MD} The DX dataset consists of 1148 entries of real cases across 461 unique diseases. We sampled 100 out of the 461 unique diseases uniformly at random to be used as the item hypothesis space. For the average case analysis, we count how many entries in the DX dataset has each of the 100 diseases as the final diagnosis, and use this as an estimate of the population distribution.

\paragraph{TS} Flodial contains both laptop-related and car-related faults. We decided to focus on the latter due to its greater variety and clearer distinction between each fault. In the dataset, each fault corresponds to a short textual description at a leaf node of a directed graph that encodes a troubleshooting flowchart. We extract these leaf descriptions and prompt GPT-5.2 Thinking to map each one to a concrete car-fault label. We then manually verify and edit the resulting labels based on the associated troubleshooting dialogues provided in Flodial. Finally, we merge or remove near-duplicate faults (e.g., “water leakage” vs. “water pump leak”) to ensure that each retained fault is distinct and readily distinguishable.

\section{Examples of Questions Asked}
\paragraph{Unnatural Questions} A typical question posed by the LLM in a game of Twenty Questions with the Common dataset resembles the following:
\begin{itemize}
    \item Is the item something a person can physically pick up with their hands (assuming average human size and strength)?
    \item Is the item man-made?
    \item Is the item an animal?
\end{itemize}
Such questions generally pertain to intrinsic properties of the items themselves.

In contrast, when the \textit{even} prompt was used, we occasionally observed the generation of more "unnatural" questions:

\begin{itemize}
    \item Is the breed's name made up of only one word?
    \item Is the name of the item longer than 7 characters?
    \item Does the item's name start with a letter from A-M (inclusive)?
\end{itemize}

While we find it difficult provide a quantitative evaluation of the ``naturalness'' of a question, intuitively, unnatural questions typically reference the item names directly, rather than the underlying items themselves. Natural questions on the other hand generally focuses on intrinsic properties of the items.
Unnatural questions can be interpreted as the LLM's attempt to optimize the question by improving the chances of obtaining an even-split, which is optimal as implied by \cref{thm:even-split-opt-sls}. These occurrences were primarily observed when GPT-4.1 was used as the LLM. We suspect this to be why a noticeable improvement in DP can be observed when \textit{even} is used over \textit{natural}, as seen from \cref{tab:dls_natural_even}. This is less noticeable in GoT and UoT, possibly due to both performing planning via lookahead. This could reduce the influence of the few ``unnatural" questions on the overall performance.

\begin{table}[ht]
\centering
\begin{tabular}{c c c c}
    \multirow{2}{*}{\textbf{Method}} & \multicolumn{3}{c}{\textbf{20Q}}\\
& \textbf{Common} & \textbf{S128} & \textbf{Breeds}\\
\multicolumn{4}{c}{\cellcolor{gray!15}\textbf{GPT 4.1 + Even split prompt}} \\
GoT& 9.4 & 9.2 & 6.4 \\
UoT & 10 & 10 & 7\\
DP & 11.7 & 12.9 & 7.8\\
DC & 12.3 & 12.6 & 9.0\\
\multicolumn{4}{c}{\cellcolor{gray!15}\textbf{GPT 4.1 + Natural prompt}} \\
GoT& 10.2 & 11.8 & 7.4\\
UoT & 11 & 13 & 9\\
DP & 13.8 & 16.2  & 7.8 \\
DC & 12.9 & 14.6 & 9.3\\
\multicolumn{4}{c}{\cellcolor{gray!15}\textbf{Qwen 2.5 72B Instruct + Even split prompt}} \\
GoT& 10.2 & 10.2 & 6.2\\
UoT & 11 & 11 & 7\\
DP & 11.9 & 14.6 & 8.0\\
DC & 12.4 & 13.6 & 8.3 \\
\multicolumn{4}{c}{\cellcolor{gray!15}\textbf{Qwen 2.5 72B Instruct + Natural prompt}} \\
GoT& 10.0 & 10.8 & 6.6\\
UoT & 11 & 12 & 8\\
DP & 12.7 & 19.2 & 8.0\\
DC & 12.7 & 17.9 & 7.8\\
\end{tabular}
\caption{\textbf{Worst case interaction length for each method using even or natural prompts}. For both Uot and GoT We set $d = 3$.}
\label{tab:dls_natural_even}
\end{table}

\paragraph{Questions in Practical Settings}
As noted previously, we claimed that the questions generated by the LLM in the MD and TS settings are reasonably natural. To illustrate this, we provide representative examples of questions sampled from GPT 4.1 in each setting.

Examples of questions sampled for Medical Diagnosis:
\begin{itemize}
    \item Is the main symptom dermatological, affecting the skin (e.g., rashes, eczema, pigmentation, lesions)?
    \item Is the disease primarily associated with an infection or inflammation (for example, hepatitis, bronchitis, pneumonia, infectious mononucleosis, sepsis, pharyngitis, gastritis)?
    \item Is the primary symptom related to the respiratory system (such as cough, shortness of breath, or chest pain)?
\end{itemize}

These questions target either the patient’s symptoms or the affected body systems, which aligns with how diagnosis is typically conducted in practice. Moreover, each question can be (i) posed to the patient with minor rephrasing, (ii) used as a mental checklist during clinical reasoning, or (iii) mapped to a diagnostic test relevant to specific diseases.

Examples of questions sampled for Troubleshooting:
\begin{itemize}
    \item Do you notice any warning lights or indicators on the dashboard when you turn the key to the 'ON' position?
    \item Are there any electrical accessories (like headlights or dashboard lights) not functioning even when the ignition is off?
    \item Does the car fail to start or crank when turning the ignition key?
\end{itemize}

These questions are easy to understand and, importantly, refer to symptoms that can be readily verified by the customer. This property makes them suitable for the context of troubleshooting.

\paragraph{Ambiguous Questions}

In our formulation of the SLS problem, we assume that (i) the codomain of $f$ is binary, and (ii) each item $s \in \mathcal{S}$ has a unique, well-defined answer under $f$, as stated in Assumption~\ref{ass:unique_ans}. However, this assumption may not hold in real-world scenarios due to the inherent ambiguities and interpretive variability commonly present in natural language. One source of such ambiguity arises when an item name refers to multiple distinct entities. For example, consider the question: Is the Titanic something a person can pick up with their hands? The answer would be no if referring to the actual ship, but yes if referring to a DVD of the movie Titanic.
While language ambiguity is a well-studied phenomenon in the field of Natural Language Processing, addressing such ambiguities remains challenging when applied to the more formal and rigorous framework of Game Theory. We believe that addressing this challenge presents a promising direction for future research.

\section{Prompts Used}
The prompts used for sampling questions (to implement the function $g$) in the weighted and unweighted variants are shown in \cref{tab:slsr_prompts,tab:wslsr_prompts} respectively. 
DP uses the same prompts for asking questions by setting $m=1$. DC additionally uses the prompt in \cref{tab:choose_prompt} to choose the question from the candidate set. The prompts used to generate answers to questions for items (to implement the function $f$) is shown in \cref{tab:class_prompts}.

\begin{table*}[ht]
\centering
\begin{tabular}{p{15cm}}
\cellcolor{gray!15}\textbf{Even Prompt 20Q}\\
We are playing the game of ask and seek, where I am the answerer and you are the questioner. \\
I have chosen one item from a list of items, and your goal is to identify the correct item from a list of items by asking a series of yes or no questions. \\
Your should try to ask as few questions as possible.\\
Here are all the items:\\
\textbf{[items]}\\
One of the items is the correct item, please design questions about those items that helps you find the right item and can only be answer by Yes or No.\\
\textit{Notably, this question should fulfill that the count of YES items and NO items are almost the same with a permissible discrepancy of no more than one!}\\
You should respond with the best \textbf{[m]} questions you can think of.\\
\cellcolor{gray!15}\textbf{Natural Prompt 20Q}\\
We are playing the game of ask and seek, where I am the answerer and you are the questioner. \\
I have chosen one item from a list of items, and your goal is to identify the correct item from a list of items by asking a series of yes or no questions. \\
Your should try to ask as few questions as possible.\\
Here are all the items:\\
\textbf{[items]}\\
One of the items is the correct item, please design questions about those items that helps you find the right item and can only be answer by Yes or No.\\
\textit{You should avoid asking unnatural questions such as ones about the number of words the item has, or any such similar questions.}\\
You should respond with the best \textbf{[m]} questions you can think of.\\
\cellcolor{gray!15}\textbf{MD}\\
You are a doctor trying to diagnose a patient. \\
The patient suffers from one of the diseases from the provided list, and your goal is to identify the correct item from this list asking a series of yes or no questions. \\
You should aim to ask as few questions as possible. \\
Here are all the diseases:
\textbf{[diseases]}\\
One of the diseases is the correct disease, please design questions about those items that help you find the right item and can only be answered by Yes or No.\\
You should respond with the best \textbf{[m]} questions you can think of.\\
\cellcolor{gray!15}\textbf{TS}\\
You are a car mechanic providing remote troubleshooting to a customer.\\
The customer's car suffers from one of the faults from the provided list, and your goal is to identify the correct fault from this list asking a series of yes or no questions.\\
You should aim to ask as few questions as possible.\\
Here are all the faults:
\textbf{[faults]}\\
One of the faults is the correct fault, please design questions about these faults that help you find the right item and can only be answered by Yes or No.\\
You should respond with the best \textbf{[m]} questions you can think of.\\
\end{tabular}
\caption{\textbf{Prompts used to sample candidate questions for unweighted variant.} The main difference between even and natural prompts are \textit{italicized}. The items are represented as a python list, e.g. [`Oppenheimer', `Alan Turing', `A Beautiful Mind'], while $m$ is an integer}
\label{tab:slsr_prompts}
\end{table*}

\begin{table*}[ht]
\centering
\begin{tabular}{p{15cm}}
\cellcolor{gray!15}\textbf{20Q}\\
We are playing the game of ask and seek, where I am the answerer and you are the questioner. \\
I have chosen one item from a list of items, and your goal is to identify the correct item from a list of items by asking a series of yes or no questions. \\
Each item has a value associated with it, and your penalty score will be the number of questions asked until you find the right item multiplied by the value of the correct item. You should try to minimize this score. \\
Here are all the items and their values:
\textbf{[items with weights]}\\
One of the items is the correct item, please design questions about those items that helps you find the right item and can only be answer by Yes or No.\\
You should avoid asking unnatural questions such as ones about the number of words the item has, or any such similar questions.\\
You should respond with the best \textbf{[m]} questions you can think of.\\
\cellcolor{gray!15}\textbf{MD}\\
You are a doctor trying to diagnose a patient.\\
The patient suffers from one of the diseases from the provided list, and your goal is to identify the correct item from this list asking a series of yes or no questions.\\
Each disease has a severity associated with it. Your penalty will be the number of questions asked multiplied by the severity value of the correct disease. You should try to minimize this penalty.\\
Here are all the diseases:
\textbf{[diseases with weights]}\\
One of the diseases is the correct disease, please design questions about those items that help you find the right item and can only be answered by Yes or No.\\
You should respond with the best \textbf{[m]} questions you can think of.\\
\cellcolor{gray!15}\textbf{TS}\\
You are a car mechanic providing remote troubleshooting to a customer.\\
The customer's car suffers from one of the faults from the provided list, and your goal is to identify the correct fault from this list asking a series of yes or no questions.\\
Each fault has a severity associated with it. Your penalty will be the number of questions asked multiplied by the severity value of the correct fault. You should try to minimize this penalty.\\
Here are all the faults:
\textbf{[faults with weights]}\\
One of the faults is the correct fault, please design questions about these faults that help you find the right item and can only be answered by Yes or No.\\
You should respond with the best \textbf{[m]} questions you can think of.\\
\end{tabular}
\caption{\textbf{Prompt used to sample candidate questions for weighted variant.} Items with weights are represented as a python dictionary, e.g. \{`Oppenheimer' : 3, `Alan Turing' : 2, `A Beautiful Mind' : 2\}, while $m$ is an integer}
\label{tab:wslsr_prompts}
\end{table*}

\begin{table*}[ht]
\centering
\begin{tabular}{p{15cm}}
\cellcolor{gray!15}\textbf{Choosing Prompt for 20Q Unweighted Variant}\\
We are playing the game of ask and seek, where I am the answerer and you are the questioner.\\
I have chosen one item from a list of items, and your goal is to identify the correct item from a list of items by asking a series of yes or no questions.\\
Your should try to ask as few questions as possible.\\
Here are all the items:\\
\textbf{[items]}\\
Here are some Yes or No question about them: \\
\textbf{[question]}\\
One of the items is the correct item, please choose one question from the list of questions that helps you find the right item.\\
\cellcolor{gray!15}\textbf{Choosing Prompt for 20Q Weighted Variant}\\
We are playing the game of ask and seek, where I am the answerer and you are the questioner.\\
I have chosen one item from a list of items, and your goal is to identify the correct item from a list of items by asking a series of yes or no questions.\\
Each item has a value associated with it, and your penalty score will be the number of questions asked until you find the right item multiplied by the value of the correct item. You should try to minimize this score.\\
Here are all the items:\\
\textbf{[items with weights]}\\
Here are some Yes or No question about them: \\
\textbf{[question]}\\
One of the items is the correct item, please choose one question from the list of questions that helps you find the right item.\\
\cellcolor{gray!15}\textbf{Choosing Prompt for MD Unweighted Variant}\\
'You are a doctor trying to diagnose a patient.\\
The patient suffers from one of the diseases from the provided list, and your goal is to identify the correct item from this list asking a series of yes or no questions.\\
Your should aim to ask as few questions as possible.\\
Here are all the diseases:
\textbf{[diseases]}\\
Here are some Yes or No question about them: 
\textbf{[question]}\\
One of the diseases is the correct disease, please choose one question from the list of questions that helps you narrow down to the right disease as quickly as possible.\\
\cellcolor{gray!15}\textbf{Choosing Prompt for TS Unweighted Variant}\\
You are a car mechanic providing remote troubleshooting to a customer.\\
The customer's car suffers from one of the faults from the provided list, and your goal is to identify the correct fault from this list asking a series of yes or no questions.\\
Your should aim to ask as few questions as possible.\\
Here are all the faults:
\textbf{[faults]}\\
Here are some Yes or No question about them: 
\textbf{[question]}\\
One of the faults is the correct fault, please choose one question from the list of questions that helps you narrow down to the right fault as quickly as possible.\\
\end{tabular}
\caption{\textbf{Prompts used by DC to choose questions.} Questions are represented as a python dictionary, e.g. \{0 : `Related to codes?', 1 : `Is it a movie?', 2 : `Is it a person?'\}. The weighted variants of the prompt for MD and TS setting is obtained by modifying the unweighted variant in a similar manner as in the 20Q setting.} 
\label{tab:choose_prompt}
\end{table*}

\begin{table*}[ht]
\centering
\begin{tabular}{p{15cm}}
\cellcolor{gray!15}\textbf{Answer Prompt}\\
Here are some \textbf{[items/diseases/car faults]}:\\
\textbf{[items]}\\
Here is a Yes or No question about them: \\
\textbf{[question]}\\
Please classify the \textbf{[items/diseases/car faults]} above based on this question. You should ensure that each item should only appear once in the final classification. If you are unsure about any item, you can mention it in the elaboration.\\
\end{tabular}
\caption{\textbf{Prompt used to get answers to questions.} We refer to the hypothesis space as items, diseases or car faults for the 20Q, MD, and TS setting respectively.}
\label{tab:class_prompts}
\end{table*}

\section{Experiment Hardware}
Running the EFG solver was done on a AMD Ryzen 5 5500 CPU with 32 GB ram. Hosting the local instance of Qwen 2.5 72B Instruct LLM was done via the vLLM\cite{kwon2023efficient} library on either 2 Nvidia H200 GPUs each with 141 GB memory, or 4 Nvidia H100 GPUs each with 80 GB memory.

\end{document}